\documentclass{article}

\usepackage{microtype}
\usepackage{graphicx}
\usepackage{subfigure}
\usepackage{booktabs} 
\usepackage{comment}

\usepackage{multicol}
\usepackage{multirow}
\usepackage{mdframed}

\usepackage{hyperref}

\usepackage{amsmath}
\usepackage{amssymb}
\usepackage{mathtools}
\usepackage{amsthm}

\usepackage{algorithm}
\usepackage{algorithmic}

\theoremstyle{plain}
\newtheorem{theorem}{Theorem}[section]

\theoremstyle{definition}
\newtheorem{definition}[theorem]{Definition}

\theoremstyle{remark}
\newtheorem{remark}[theorem]{Remark}

\usepackage[textsize=tiny]{todonotes}

\definecolor{lm_purple_low}{RGB}{240,240,248}
\definecolor{lm_purple}{RGB}{227,227,240}

\definecolor{mygray}{gray}{0.95}

\usepackage{blindtext}
\usepackage{titlesec}

\usepackage[preprint]{neurips_2025}

\usepackage[utf8]{inputenc} 
\usepackage[T1]{fontenc}    
\usepackage{hyperref}       
\usepackage{url}            
\usepackage{booktabs}       
\usepackage{amsfonts}       
\usepackage{nicefrac}       
\usepackage{microtype}      
\usepackage{xcolor}         

\title{Fine-tuning Flow Matching Generative Models with Intermediate Feedback}

\author{%
  Jiajun Fan$^1$, Chaoran Cheng$^1$, Shuaike Shen$^{2}$,  Xiangxin Zhou$^3$, Ge Liu$^1$ \\
$^1$ University of Illinois Urbana-Champaign, $^2$ Carnegie Mellon University.\\
$^3$ University of Chinese Academy of Sciences. \\
$^1$ \texttt{\{jiajunf3, chaoran7, geliu\}@illinois.edu} \\
$^{2,3}$ \texttt{shuaikes@andrew.cmu.edu, zhouxiangxin1998@gmail.com} \\
}

\begin{document}

\maketitle

\begin{abstract}
Flow-based generative models have shown remarkable success in text-to-image generation, yet fine-tuning them with intermediate feedback remains challenging, especially for continuous-time flow matching models. Most existing approaches solely learn from outcome rewards, struggling with the credit assignment problem. Alternative methods that attempt to learn a critic via direct regression on cumulative rewards often face training instabilities and model collapse in online settings. We present AC-Flow, a robust actor-critic framework that addresses these challenges through three key innovations: (1) reward shaping that provides well-normalized learning signals to enable stable intermediate value learning and gradient control, (2) a novel dual-stability mechanism that combines advantage clipping to prevent destructive policy updates with a warm-up phase that allows the critic to mature before influencing the actor, and (3) a scalable generalized critic weighting scheme that extends traditional reward-weighted methods while preserving model diversity through  Wasserstein regularization. Through extensive experiments on Stable Diffusion 3, we demonstrate that AC-Flow achieves state-of-the-art performance in text-to-image alignment tasks and generalization to unseen human preference models. Our results demonstrate that even with a computationally efficient critic model, we can robustly finetune flow models without compromising generative quality, diversity, or stability.
\end{abstract}

\section{Introduction}
\label{sec: introduction}

Large-scale generative models have revolutionized artificial intelligence, with flow matching (FM) emerging as a particularly promising framework that has achieved state-of-the-art (SOTA) performance in generating high-fidelity images \citep{sd3} and complex molecular structures like proteins \citep{se3}. However, how to effectively incorporate process feedback to improve these models remains challenging. Most current approaches predominantly rely on outcome rewards \citep{guo2025deepseek_r1,ddpo,iclr_rwr}, which suffer from credit assignment problem: when the same reward signal is applied uniformly across all intermediate states of the generative trajectory, the model receives no intermediate feedback about which specific steps in the generation process were instrumental in achieving high-quality results \citep{anonymous2025derivativefree}. The ability to accurately evaluate and optimize intermediate states is thus critical for precise control over the generative process \citep{ver_step_by_step,instructed_rlhf}, yet remains a significant challenge for fine-tuning flow matching models, where the continuous-time ODE-based dynamics and high-dimensional state space (e.g., image) make intermediate state value estimation particularly difficult.

\begin{figure*}[!ht]
	\centering
 \includegraphics[width=\linewidth]{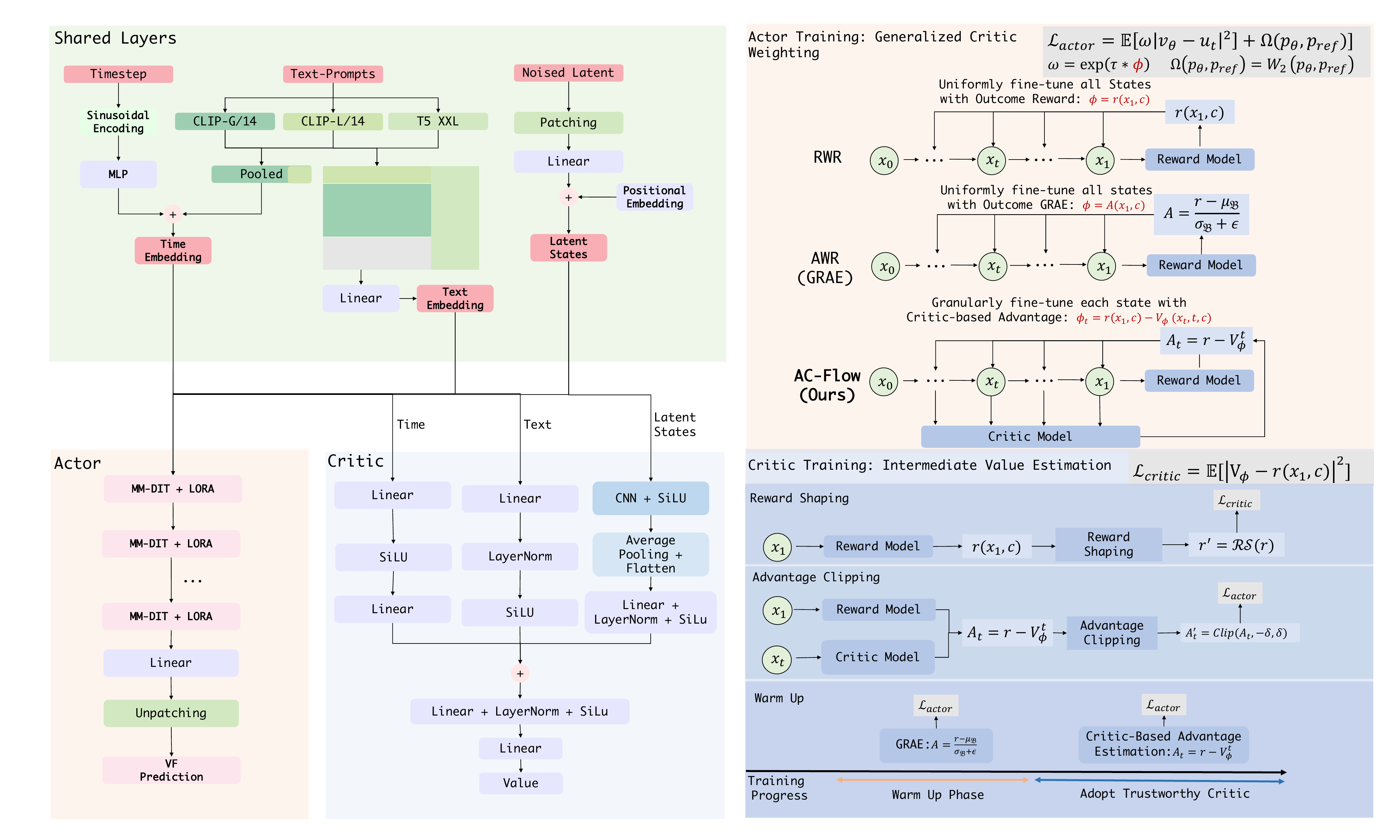}
	\caption{General Framework of Our AC-Flow including our Model Architecture (Left) and Training Techniques (Right). We also include pipeline of outcome-driven methods like RWR \citep{iclr_rwr} and AWR \citep{awr} in actor part (Right), to highlight the differences and our unique contribution.}
    \label{fig: general framework}
\end{figure*}

To obtain process feedback, recent works have shown promising results by learning critics through direct value regression in both offline \citep{anonymous2025derivativefree} and diffusion \citep{xiangxin_stable} settings. However, extending these approaches to online fine-tuning of flow matching models with self-generated data presents fundamental challenges. First, vanilla actor-critic methods, when naively applying value regression in online settings, can trigger catastrophic instabilities and create destructive feedback loops \citep{offline_to_online} where both actor and critic losses escalate uncontrollably (see Fig. \ref{fig: exp stable critic training}). Second, training with self-generated data frequently leads to model collapse, causing catastrophic loss of generative diversity \citep{model_collapse,iclr_rwr}. Third, the computationally intractable generative likelihood in FM and the ill-defined evidence lower bound (ELBO) \citep{fmgm,iclr_rwr} prevent the direct application of widely-adopted policy gradient methods \citep{ddpo,xiangxin_stable,gae} in ODE-based FM models. These challenges are particularly acute in large-scale text-to-image generation, where high-dimensional state spaces and sparse rewards make stable and efficient online actor-critic training exceptionally difficult without proper stabilization techniques \citep{offline_to_online} (See Fig. \ref{fig: exp stable critic training}).

To tackle these challenges, we present AC-Flow, an online actor-critic framework for fine-tuning FM models that achieves stable improvement from intermediate feedback (e.g., advantages for each intermediate state as Fig. \ref{fig: general framework}). Our approach introduces several synergistic innovations: (1) a computationally efficient critic architecture combined with reward shaping and stabilization techniques that enables robust intermediate state value estimation without excessive computational overhead; (2) a generalized critic weighting scheme that unifies and extends traditional reward-weighted methods  to leverage intermediate state values while maintaining low variance through advantage estimation, inspired by GAE \citep{gae}; (3) a dual-stability mechanism that combines advantage clipping with warm-up strategies to prevent destructive policy updates while allowing the critic to mature; and (4) an effective Wasserstein-2 regularization scheme that maintains generative diversity while enabling stable learning from self-generated data.  In general, our primary contributions can be summarized as follows:

\begin{enumerate}
\item \textbf{Stable Intermediate State Value Evaluation.} As Fig. \ref{fig: general framework}, we propose a computationally efficient critic architecture that successfully learns the value estimation of intermediate states for continuous-time flow matching by combining simple value regression \citep{xiangxin_stable} with novel stabilization techniques. Our approach uses reward shaping \citep{shape_rewards} to normalize reward scales and enable robust critic convergence while avoiding training collapse (See Fig. \ref{fig: exp stable critic training}).

\item \textbf{Generalized Critic Weighting.} To handle the computational intractability of policy gradient in ODE-based flow matching, we develop a unified framework that extends traditional reward-weighted methods through generalized critic weighting (GCW). Unlike existing outcome-driven approaches \citep{ddpo,Raft}, our method leverages intermediate state values for granular fine-tuning while preserving generative diversity through  Wasserstein regularization.

\item \textbf{Robust Actor-Critic Framework.} We introduce a comprehensive stability mechanism that combines reward shaping, advantage clipping, and warm-up strategies to achieve online actor-critic fine-tuning without training collapse or reward hacking (Fig. \ref{fig: exp stable critic training} ,\ref{fig: main com complex semantic alignment}). This approach ensures gradual incorporation of critic feedback into policy optimization while maintaining computational efficiency, making it readily applicable to  large-scale models.

\item \textbf{Empirical Validation.} Through extensive experiments with Stable Diffusion 3 \citep[SD3]{sd3}, we demonstrate state-of-the-art performance in text-to-image generation, achieving superior semantics alignment \citep{clip,image_reward}, and generalization to unseen human preference models \citep{hpsv2} while maintaining generative diversity (See Tab. \ref{tab:main_results}, Figs. \ref{fig: main com complex semantic alignment}).
\end{enumerate}

\section{Related Work}
\label{sec: related work}

\textbf{Fine-tuning from Human Feedback.} Current approaches to reinforcement learning from human feedback (RLHF) typically fall into two categories: reward-based methods \citep{adjoint_matching,diff_rewards} and direct preference optimization \citep[DPO]{dpo}. While DPO enables direct model fine-tuning using filtered datasets, it cannot incorporate reward models. Additionally, existing reward-based approaches like Adjoint-Matching \citep{adjoint_matching} and DRaFT \citep{diff_rewards} require differentiable rewards, limiting their applicability to practical cases where reward gradients is unable to obtain such as rule-based or rank-based reward \citep{grpo}. Despite the success of process-based RLHF in language models \citep{ver_step_by_step,zhang2024restmcts}, extending these benefits to online fine-tuning of FM models while supporting arbitrary reward functions remains an open challenge.

\textbf{Fine-tuning Flow Matching Generative Models.} While recent work has advanced fine-tuning methods for diffusion models through both policy gradient approaches \citep{xiangxin_stable,ddpo,dpok} and reward-based techniques \citep{Raft,iclr_rwr}, these methods face fundamental limitations when applied to fine-tune FM models with process feedback. Policy gradient methods like DDPO \citep{ddpo}, despite their success in diffusion models, struggle with flow matching due to the computational intractability of estimating transition probabilities \citep{fmgm,iclr_rwr}. Besides, reward-based approaches like RAFT \citep{Raft} and RWR \citep{iclr_rwr}, which rely solely on outcome rewards, can provide misleading guidance for intermediate state optimization \citep{ver_step_by_step}.

\textbf{Value Estimation in Fine-tuning Generative Models.} A direct approach to achieving RLHF with process feedback is to employ actor-critic architectures, similar to PPO \citep{instructed_rlhf}, using generalized advantage estimation \citep[GAE]{gae}. Recent work has demonstrated significant progress in learning stable critic models for both diffusion models in offline settings \citep{anonymous2025derivativefree} and stochastic differential equations \citep{xiangxin_stable}. However, the challenge of achieving stable online actor-critic fine-tuning for large-scale continuous-time ODE-based flow matching models, such as SD3, remains unsolved. This is particularly challenging due to the inherent instabilities in online value learning \citep{offline_to_online} and the potential for catastrophic forgetting in online learning scenarios \citep{forget}.

\section{Methodology}
\label{sec: method}

\subsection{Problem Setup}
We formalize the reinforcement learning (RL) fine-tuning problem for continuous-time flow matching generative models as follows (See App. \ref{app: background} for more backgrounds of RL and FM). Let $\theta \in \mathbb{R}^m$ parameterize a conditional flow matching generative model  $p_\theta(x_1 \mid c)$, where $x_1 \in \mathcal{X} \subseteq \mathbb{R}^d$ denotes the generated sample (e.g., an image) and $c \in \mathcal{C}$ is the conditioning context (e.g., a text prompt). The model defines a continuous-time trajectory $\left\{x_t\right\}_{t \in[0,1]}$, governed by the ODE $\frac{d x_t}{d t}=v_\theta\left(t, x_t, c\right), \quad x_0 \sim p_0\left(x_0\right)$, where $v_\theta(t, x_t, c)$ is the learned vector field transporting an initial distribution $p_0\left(x_0\right)$ to the target distribution $p_\theta\left(x_1 \mid c\right)$.   The conditional flow matching \citep[CFM]{fmgm}  trains $v_\theta$ to align with a reference vector field $u_t\left(x_t \mid x_1, c\right)$, which defines a \textit{conditional probability path} $p\left(x_t \mid x_1, c\right)$ for individual target samples $x_1 \sim q\left(x_1 \mid c\right)$. CFM objective minimizes:

\begin{equation}
    \mathcal{L}_{\mathrm{CFM}}(\theta)=\mathbb{E}_{t \sim \mathcal{U}[0,1], c \sim p(c), x_1 \sim q\left(x_1 \mid c\right), x_t \sim p\left(x_t \mid x_1, c\right)}\left\|v_\theta(t, x_t, c)-u_t\left(x_t \mid x_1, c\right)\right\|^2,
\end{equation}

where, $p(c)$ is the sampling distribution of conditioning context, and $u_t\left(x_t \mid x_1, c\right)$ ensures $p\left(x_t \mid x_1, c\right)$ evolves smoothly from $p_0\left(x_0\right)$ to $q\left(x_1 \mid c\right)$ \citep{otcfm}.

\textbf{Generative Process.} The learned $v_\theta$ induces a deterministic push-forward map $\Phi_{0 \rightarrow 1}^c$, transforming $p_0\left(x_0\right)$ into the target distribution  $q\left(x_1 \mid c\right)=\Phi_{0 \rightarrow 1}^c \# p_0\left(x_0\right)$, where $\#$ denotes the push-forward operation. To sample from $p_\theta(x_1 \mid c)$, we solve ODE: $x_1=x_0+\int_0^1 v_\theta\left(t, x_t, c\right) d t$.

\textbf{RL Objective.} Given a reward function $r: \mathcal{X} \times \mathcal{C} \rightarrow \mathbb{R}$ that evaluates the quality of generated samples, our goal is to adapt $\theta$ such that $p_\theta(x_1 \mid c)$ maximizes the expected reward:

\begin{equation}
\max _\theta \mathbb{E}_{x_1 \sim p_\theta(x_1 \mid c), c \sim p(c)}[r(x_1, c)].
\end{equation}




\subsection{Fine-tuning with Outcome  Reward}

Prior work in fine-tuning large generative models has largely focused on outcome-driven optimization \citep{ddpo,dpok,Raft,ReFT,iclr_rwr}, where policy updates depend exclusively on scalar rewards assigned to final states $x_1$, leading to \colorbox{green!20}{uniform credit assignment} (i.e., $\omega(x_{t\in[0,1]},c)=\omega(x_1,c)$) for each state in the generative trajectory. These approaches treat the generative trajectory $\left\{x_t\right\}_{t \in[0,1]}$ as a monolithic entity, optimizing the entire path based on the terminal reward $r\left(x_1, c\right)$. For instance, online reward weighting methods \citep{iclr_rwr} assign trajectory weights proportional to the exponential of the final reward as $w\left(x_1, c\right) \propto \exp \left(\tau r\left(x_1, c\right)\right)$, where $\tau >0$. For each round, the policy update minimizes a weighted alignment objective between the learned vector field $v_\theta(t, x_t, c)$ and the reference flow $u_t\left(x_t \mid x_1, c\right)$:

\begin{equation}
\label{equ: rwr loss}
\mathbb{E}_{t \sim \mathcal{U}[0,1], c \sim p(c), x_1 \sim p_\theta^{(n-1)}(x_1 \mid c), x \sim p\left(x_t \mid x_1, c\right)}\left[\colorbox{green!20}{$w\left(x_1, c\right)$}\left\|v_\theta(t, x_t, c)-u_t\left(x_t \mid x_1, c\right)\right\|^2\right] ,
\end{equation}

where $x_1$ is sampled by the learned policy at round $n-1$, $p_\theta^{(0)}(x_1 \mid c)=p_{ref}(x_1 \mid c)$. When training converges for each round, this induces a multiplicative shift in the learned policy distribution \citep{awr,iclr_rwr}:

\begin{theorem}[Policy Update in Online Reward-Weighting]
\label{Theorem: Policy Update in Online Reward Weighting}
    Under ideal conditions, if at round $n$ the parameters $\theta$ perfectly minimize online reward-weighting loss as Equ. \eqref{equ: rwr loss}, then the learned policy distribution satisfies:
\begin{equation}
    p_\theta^{(n)}(x_1 \mid c) \propto w(x_1, c) p_\theta^{(n-1)}(x_1 \mid c).
\end{equation}
\end{theorem}

\begin{remark}
    \textit{While these guarantees hold in theory, they rely on a critical assumption: all intermediate states $x_t$ in a trajectory contribute equally to the final reward.}
\end{remark}

\textbf{Limitation of Outcome-Driven Methods.} In practice, outcome-driven methods suffer from process reward ambiguity, namely the inability to attribute feedback to specific steps in the generation process. This manifests as two interrelated challenges: (1) sparse and noisy learning signals, where flawed intermediate states (e.g., misaligned denoising at $t=0.3$ ) irreversibly degrade $x_1$, yet uniform weighting $w\left(x_1, c\right)$ fails to isolate their impact, causing early errors to propagate uncorrected while beneficial steps lack targeted reinforcement; and (2) catastrophic entropy collapse, where in online settings, noisy or sparse final rewards \textit{amplify gradient variance}, leading to poorly timed updates (driven by overestimated or underestimated feedbacks) that destabilize training and collapse policies into low-diversity modes \citep{model_collapse}. For instance, in text-to-image synthesis with SD3, failing to position ``a red apple left of a green backpack'' at $t=0.5$ may yield an incoherent image, but \textit{outcome-driven methods uniformly penalizes all steps, leaving the model unable to discern which phase caused the error.} Therefore, we need to estimate the value of intermediate states, and achieve fine-tuning from more granular feedback for each generative steps with online RL.

\subsection{Critic: Value Estimation of Intermediate State}

A critical challenge in online fine-tuning is to properly evaluate the 
\emph{intermediate} states $\left(x_t, c\right)$ visited by the flow matching model. Unlike reward weighting method-where only a final state reward $r\left(x_T, c\right)$ is observed-here we seek to learn a 
\emph{value function} $V_\phi$ that approximates the expected sum of rewards from any partial state $\left(x_t, c\right)$ forward to completion. This provides more granular feedback for policy improvements and alleviates the sparse-reward problem inherent in purely outcome-driven methods.

\textbf{Value Estimation.} Let $\left\{x_s\right\}_{s=t}^T$ be a trajectory generated by the flow model from time $t$ through the terminal time $T=1$.  We define the \emph{return} for the partial trajectory starting at $\left(x_t, c\right)$ by summing all future rewards (i.e., the cumulative reward): $G_t = \Sigma r=r(x_T,c)$. Our goal is to learn a function $V_\phi(x, t, c)$ that regresses onto the expectation of $G_t$. Concretely, we minimize the mean-squared error (MSE):

\begin{equation}
    \mathcal{L}_{\text {value }}(\phi)=\mathbb{E}\left[\left(V_\phi\left(x_t, t, c\right)-G_t\right)^2\right].
\end{equation}

The critic loss can be estimated as  $\frac{1}{n} \sum_{j=1}^n\left(V_\phi\left(x_t^j, t, c\right)-r\left(x_1^j, c\right)\right)^2$, where $x_t^j$ is sampled from $j$th trajectory at $t$. In practice, we will sample $\left(x_t, c\right)$ at various $t$ for stable learning \citep{sd3}. This regression view is analogous to standard value-function learning in reinforcement learning, except that our ``state'' is the continuous-time flow location $(x_t,t,c)$ and the future spans $s \in[t, 1]$. By aligning $V_\phi\left(x_t, t, c\right)$ with the sum of subsequent rewards, the model gains a local measure of how promising an intermediate state is. This local measure is especially critical for fine-tuning continuous flows, as it helps the policy discern which partial states lead to high-quality final outputs. We now can adopt $A=r(x_1,c)-V(x_t,t,c)$ instead of outcome-reward $r(x_1,c)$ to estimate the intermediate weight for policy learning in Equ. \eqref{equ: rwr loss}.

\textbf{Reward Shaping.} We adopt reward shaping (e.g., reward scaling \citep{r2d2}) to stabilize and accelerate critic learning by normalizing rewards to mitigate scale-related training instability. For a batch of rewards $\left\{r_i\right\}_{i=1}^N$, we apply min-max scaling \citep{min_max_normal} as:

\begin{equation}
    \mathcal{R S}\left(r_i\right)=\frac{r_i-\min \left(\left\{r_i\right\}\right)}{\max \left(\left\{r_i\right\}\right)-\min \left(\left\{r_i\right\}\right)+\epsilon},
\end{equation}

where $\epsilon>0$ prevents division by zero. For CLIP score-based rewards \citep{clip}, this ensures $r_i^{\prime} \in$ $[0,1]$, reducing variance in policy updates.

\textbf{Advantage Clipping and Critic Warm-Up.} To stabilize training against overestimated or high-variance value estimates, we integrate a dual-stability mechanism: 1) the advantage clipping mitigates aggressive policy updates by truncating extreme advantage values. For advantage estimates $A_t=r-V_\phi$, we compute $A_t^{\text {clip }}=\operatorname{clip}\left(A_t,-\delta, \delta\right)$, where $\delta>0$ (e.g., $\delta=5.0$ ) thresholds gradients, curtailing outliers that could destabilize the actor.
2) the critic warm-up addresses initial critic inaccuracy by deferring its use in policy updates. During the first $k$ steps (e.g., $k=500$ ), we train $V_\phi$ exclusively while estimating advantages via group-relative advantages estimation (GRAE) over a batch $\mathcal{B}: A_t^{\text {group }}=\frac{r-\mu_{\mathcal{B}}}{\sigma_{\mathcal{B}}+\epsilon}$, following GRPO \citep{grpo} of DeepSeek-R1 \citep{guo2025deepseek_r1}. This decouples early-stage actor learning from an underfitted critic. After Critic warm-up ( $t \geq k$ ), we transition to critic-guided advantages $A_t=r-V_\phi$, ensuring stable policy updates as $V_\phi$ matures.

\subsection{Training Actor via Generalized Critic Weighting}
\label{sec: Generalized Critic Weighting}

While online reward weighting provides a simple but effective mechanism to guide flow matching using only outcome rewards, it treats all intermediate states in the same trajectory identically. This \emph{credit assignment} challenge motivates introducing a \emph{critic} model to more precisely estimate how promising each intermediate state and action is. Inspired by generalized advantage estimation \citep[GAE]{gae}, we derive a family of \emph{generalized critic weighting} methods that leverage different critic forms (e.g., value function, advantage function, $Q$-function) to refine how each intermediate state in the generative trajectory is weighted.

\textbf{Generalized Critic Weighting.} We now adapt these ideas to flow matching fine-tuning, allowing the weighting function $w_\phi\left(x_t, t, c\right)$ to depend on a learned critic $\phi$ as: $ w_\phi\left(x_t, t, c\right)=\exp \left(\tau \phi\left(x_t, t, c\right)\right)$. Concretely, consider a \emph{critic} $\phi\left(x_t, t, c\right)$ that approximates any of the following quantities for the partial trajectory starting at $\left(x_t, t, c\right)$ :

\begin{equation*}
    \phi\left(x_t, t, c\right) \approx \begin{cases}\text { Outcome Reward: } &  r\left(x_1, c\right), \\ \text { GRAE: } & \frac{r-\mu_{\mathcal{B}}}{\sigma_{\mathcal{B}}+\epsilon},\\  \text { Advantage function: } & A^\pi\left(x_t, a_t, t, c\right), \\ 
    \end{cases}
\end{equation*}

Each such choice provides a different lens on \emph{how} valuable the state is or how much \emph{improvement} it can yield. Introducing \colorbox{blue!20}{intermediate state-wise credit assignment} $w_\phi(x_t,t,c)$ to replace the outcome-driven weights in Equ. \eqref{equ: rwr loss} yields the generalized critic weighting loss $\mathcal{L}_{\text{online-gcw}}(\theta)$:

\begin{equation}
\label{equ: gcw loss}
    \mathbb{E}_{t \sim \mathcal{U}[0,1], c \sim p(c), x_1 \sim p_\theta^{(n-1)}(x_1 \mid c), x_t \sim p_t\left(x_t \mid x_1, c\right)}\left[\colorbox{blue!20}{$w_\phi\left(x_t,t, c\right)$}\left\|v_\theta\left(x_t,t, c\right)-u_t\left(x_t \mid x_1, c\right)\right\|^2\right].
\end{equation}

Based on that, each intermediate state $x_t$ is re-weighted according to $\phi$. If $\phi$ represents a cumulative reward (or value), high-return states are amplified; if $\phi$ captures advantage, states that outperform their baseline get more weight than under-performing ones. In this paper,  we adopt $\phi = A  = r(x_1,c) - V(x_t,t,c)$ as our advantage estimation for generalized critic weighting for stable learning and low gradient variance, akin to RL policy gradient methods  \citep{gae}, while value-based critics provide stable baselines for long-term credit propagation.

\subsection{Fine-tuning Flow Matching Models with Wasserstein Regularization}

Similar to many previous weighting/rank based method \citep{Raft,iclr_rwr}, directly apply online generalized critic weighting may cause the entropy collapse problem. Specifically, while online generalized critic weighting  steers flow models toward high-reward regions, repeated re-weighting risks entropy collapse-degeneration into a near-deterministic distribution concentrated on a single reward-maximizing mode \citep{ddpo}. 

To balance reward pursuit with generative diversity, we penalize deviations from a pre-trained reference model $p^{\text {ref. }}$

\begin{equation}
    \mathcal{L}_{\text {online-reg }}(\theta)=\mathcal{L}_{\text {online-gcw }}(\theta)+\alpha \cdot \Omega\left(p_\theta^{(n)}, p^{\mathrm{ref}}\right),
\end{equation}

where $\Omega$ quantifies the distance between $p_\theta^{(n)}$ and $p_\theta^{\text {ref }}$, and $\alpha$ governs the exploration-exploitation trade-off. Traditional KL divergence is intractable for continuous-time flow matching. Some related work \citep{iclr_rwr,flow_q} have tried to use the Wasserstein-2  distance, which is more computation efficient. Direct $W_2$ computation is prohibitive, but a bound linking vector field differences to $W_2$ has been found in \citep{iclr_rwr,flow_q}:

\begin{theorem}[Bound on $W_2$ for Flow Matching]
   Let $v^{\theta_1}, v^{\theta_2}$ be vector fields for flow models $\theta_1, \theta_2$, with $v^{\theta_2}$ L-Lipschitz in $x$. Then:
\begin{equation}
    W_2^2\left(p_1^{\theta_1}, p_1^{\theta_2}\right) \leq e^{2 L} \int_0^1 \mathbb{E}_{x \sim p^{\theta_1},c \sim p(c)}\left[\left\|v^{\theta_1}-v^{\theta_2}\right\|^2\right] d s
\end{equation}
\end{theorem}

Approximating $W_2^2\left(p_\theta, p_{\theta_{ref}}\right)$ via Monte Carlo estimation, we penalize deviations from the reference model's vector field $v^{\theta_{ref}}$ as $\Omega\left(\theta ; \theta_{ref}\right)=\int_0^1 \mathbb{E}_{x, c}\left[\left\|v^\theta-v^{\theta_{ref}}\right\|^2\right] d s$, which yields the regularized loss: $\mathcal{L}_{\mathrm{reg}}(\theta)=\mathcal{L}_{\text {online }}(\theta)+\alpha \cdot \Omega\left(\theta ; \theta_{ref}\right)$. This constrains policy updates to stay near the reference model's behavior, preventing collapse while enabling reward-driven refinement.

\subsection{Online Actor-Critic with Wasserstein Regularization for Flow Matching}
\label{sec: Fine-tuning Flow Matching via Online Actor-Critic with Wasserstein Regularization}

We now bring together the preceding components—generalized critic weighting, value estimation, and Wasserstein regularization—into a unified actor-critic algorithm for online fine-tuning of flow matching. The actor (flow model $\theta$) adjusts its parameters to produce trajectories that maximize expected return under a critic-based weighting, while the critic (value function $\phi$) learns to estimate future rewards for each partial state. A $W_2$ penalty constrains model updates away from collapsing, ensuring stable training. Formally, the critic is optimized by minimizing:
\begin{equation}
    \mathcal{L}_{\text{critic}}(\phi)=\mathbb{E}_{t \sim \mathcal{U}[0,1], c \sim p(c), x_1 \sim p_\theta^{(n-1)}(x_1 \mid c), x_t \sim p(x_t \mid x_1, c)}\left[\left(V_\phi(x_t, t, c)-\mathcal{RS}(r(x_1, c))\right)^2\right]
\end{equation}
where $\mathcal{RS}$ denotes reward shaping through min-max normalization. The actor is updated using:
\begin{align}
    \mathcal{L}_{\text{actor-reg}}(\theta)=&\mathbb{E}_{t, x_1, c, x_t}\left[w_\phi(x_t, t, c)\left\|v_\theta(t, x_t, c)-u_t(x_t \mid x_1, c)\right\|^2\right]\\
    &+\alpha \cdot \mathbb{E}_{x_t, t, c}\left[\|v_\theta(t, x_t, c)-v_{\theta_{\text{ref}}}(t, x_t, c)\|^2\right]
\end{align}
with samples drawn from the same distribution as the critic. The weight function $w_\phi(x_t, t, c)$ incorporates our dual-stability mechanism through a strategic warm-up process. During the initial $k$ training steps, we employ group-relative advantage estimation with $w_\phi(x_t, t, c)=\exp(\tau \cdot A_t^{\text{group}})$ where $A_t^{\text{group}}=\frac{\mathcal{RS}(r(x_1, c))-\mu_\mathcal{B}}{\sigma_\mathcal{B}+\epsilon}$, allowing the critic to mature before influencing policy updates. After warm-up, we transition to critic-guided advantages with $w_\phi(x_t, t, c)=\exp(\tau \cdot A_{\text{clip}})$ where $A_{\text{clip}}=\text{clip}(\mathcal{RS}(r(x_1, c))-V_\phi(x_t, t, c),-\delta, \delta)$. This advantage clipping prevents destructive policy updates by truncating extreme advantage values. Combined with Wasserstein regularization (controlled by $\alpha$), our framework achieves a balanced policy optimization that leverages intermediate state evaluation while preventing distribution collapse—a critical advancement over previous approaches that relied solely on outcome rewards.

\section{Experiment}
\label{sec: experiment}

\subsection{Experimental Setup}
We empirically evaluate our AC-Flow on large-scale flow matching models, particularly Stable Diffusion 3 \citep[SD3]{sd3}, to assess its effectiveness in stable optimization and intermediate state value estimation. Our experiments address three key questions: (1) whether our method can effectively learn to evaluate intermediate states and stabilize actor-critic training compared to methods relying solely on value function regression, (2) whether our approach outperforms baseline fine-tuning methods that only utilize final rewards in terms of generation quality, generalization ability and diversity, and (3) how effectively our method handles challenging text-to-image alignment tasks involving complex semantics, spatial relationships, and numerical specifications. We adopt $\phi\left(x_t, t, c\right)=r(x_1,c)-V(x_t,t,c)$, warm steps $k=500$, advantage clip $\delta=5$ and $\alpha=\tau=1$ for all experiments without the need of hyperparameter tuning. See App. \ref{app: hyper-paramter}  and App. \ref{app: Experimental Details} for more experimental details.

\subsection{Main Results}

\begin{table}[!t]
\caption{Performance and Diversity comparison of different fine-tuning methods on text-image alignment using SD3 for the DrawBench prompt datasets \citep{draw_bench} as DPOK \citep{dpok}. \colorbox{violet!20}{Best scores} are highlighted in violet, \colorbox{teal!15}{second-best} in teal. For AWR, we adopt group relative advantage estimation (GRAE) as GRPO \citep{grpo}. All fine-tuning methods trained with CLIP Score \citep{clip}, while we adopt Diversity Score based on pair-wise distance of Clip Embeddings as \citep{iclr_rwr,adjoint_matching}, and HPS V2 from \citep{hpsv2} and ImageReward from \citep{image_reward}. See App. \ref{app: Evaluation Metrics} for more details. We report standard errors over three seeds.}
\label{tab:main_results}
\begin{center}
 \resizebox{0.8\textwidth}{!}{
\begin{tabular}{lcccc}
\toprule
\multirow{2}{*}{\textbf{Method}} & \multicolumn{2}{c}{\textbf{Task Metrics}} & \multicolumn{2}{c}{\textbf{Human Preference}} \\
\cmidrule(lr){2-3} \cmidrule(lr){4-5}
 & \textbf{CLIPScore $\uparrow$} & \textbf{DiversityScore $\uparrow$} & \textbf{HPS v2 $\uparrow$} & \textbf{ImageReward $\uparrow$} \\
\midrule
\multicolumn{5}{c}{\textit{Base Model}} \\
\midrule
SD3 \citep{sd3} & 29.37${\pm 0.23}$ & \colorbox{violet!20}{\textbf{4.77}${\pm 0.14}$} & 27.67${\pm 0.88}$ & 0.13${\pm 0.01}$ \\
\midrule
\multicolumn{5}{c}{\textit{Our Methods}} \\
\midrule
AC-Flow (Ours) & \colorbox{violet!20}{\textbf{32.93}${\pm 0.11}$} & 2.77${\pm 0.07}$ & \colorbox{violet!20}{\textbf{30.59}${\pm 0.28}$} & \colorbox{violet!20}{\textbf{1.20}${\pm 0.01}$} \\
AC-Flow (w/o W2) & 30.43${\pm 0.23}$ & 2.49${\pm 0.08}$ & 27.65${\pm 0.17}$ & 0.91${\pm 0.02}$ \\
\midrule
\multicolumn{5}{c}{\textit{Other Fine-tuning Methods}} \\
\midrule
Diffusion-DPO \citep{diff_dpo} & 30.24${\pm 0.11}$ & \colorbox{teal!15}{4.37${\pm 0.09}$} & 28.21${\pm 0.11}$ & 0.85${\pm 0.04}$ \\
RAFT \citep{Raft} & 29.73${\pm 0.29}$ & 2.14${\pm 0.06}$ & 26.87${\pm 0.10}$ & 0.79${\pm 0.04}$ \\
ReFT \citep{ReFT} & 29.95${\pm 0.36}$ & 2.24${\pm 0.08}$ & 27.23${\pm 0.11}$ & 0.92${\pm 0.03}$ \\
RWR+W2 \citep{rwr} & 30.56${\pm 0.15}$ & 2.54${\pm 0.07}$ & 27.78${\pm 0.23}$ & 1.03${\pm 0.02}$ \\
AWR+W2 \citep{awr} & \colorbox{teal!15}{30.97${\pm 0.33}$} & 2.73${\pm 0.08}$ & \colorbox{teal!15}{28.29${\pm 0.23}$} & \colorbox{teal!15}{1.05${\pm 0.08}$} \\
\bottomrule
\end{tabular}
}
\end{center}
\end{table}

Our experimental results in Table \ref{tab:main_results} demonstrate the strong performance of our actor-critic framework across multiple dimensions. Most notably, our method achieves state-of-the-art CLIPScore performance, substantially outperforming both the baseline SD3 model and other fine-tuning approaches including RAFT, ReFT, diffusion-DPO and reward-weighted methods. This significant improvement in CLIPScore indicates that our method generates images that achieve better semantic alignment with the given text prompts, validating the effectiveness of our intermediate state evaluation and actor-critic learning framework. Table \ref{tab:com_time} details computational cost of different methods, demonstrating our critic model does not introduce too much computational overhead (i.e., 2GB more GPU Memory, 2 hours more running time compared to methods using outcome-reward like \citep{iclr_rwr}).

A key strength of our approach is its strong generalization capability. While all methods were trained using CLIP reward, our framework shows remarkable performance on unseen human preference metrics - HPS v2 and ImageReward. Our method achieves the highest scores on both metrics, significantly outperforming other approaches. This cross-metric generalization suggests that our actor-critic framework learns fundamental aspects of text-image alignment rather than overfitting to the training objective. The consistent performance across different human preference metrics validates that our method captures improvements in generation quality that align with human preferences.

Most significantly, AC-Flow achieves SOTA performance across generation quality metrics and human preference evaluation while maintaining sample diversity. While exhibiting slightly lower diversity than the baseline SD3, our method substantially outperforms recent fine-tuning approaches in navigating the quality-diversity trade-off.  Ablation studies further validate that W2 regularization simultaneously enhances both quality metrics and sample diversity.

\subsection{Stabilizing Online Actor-Critic Fine-tuning}

\begin{figure}[!t]
	\centering
 \includegraphics[width=0.75\linewidth]{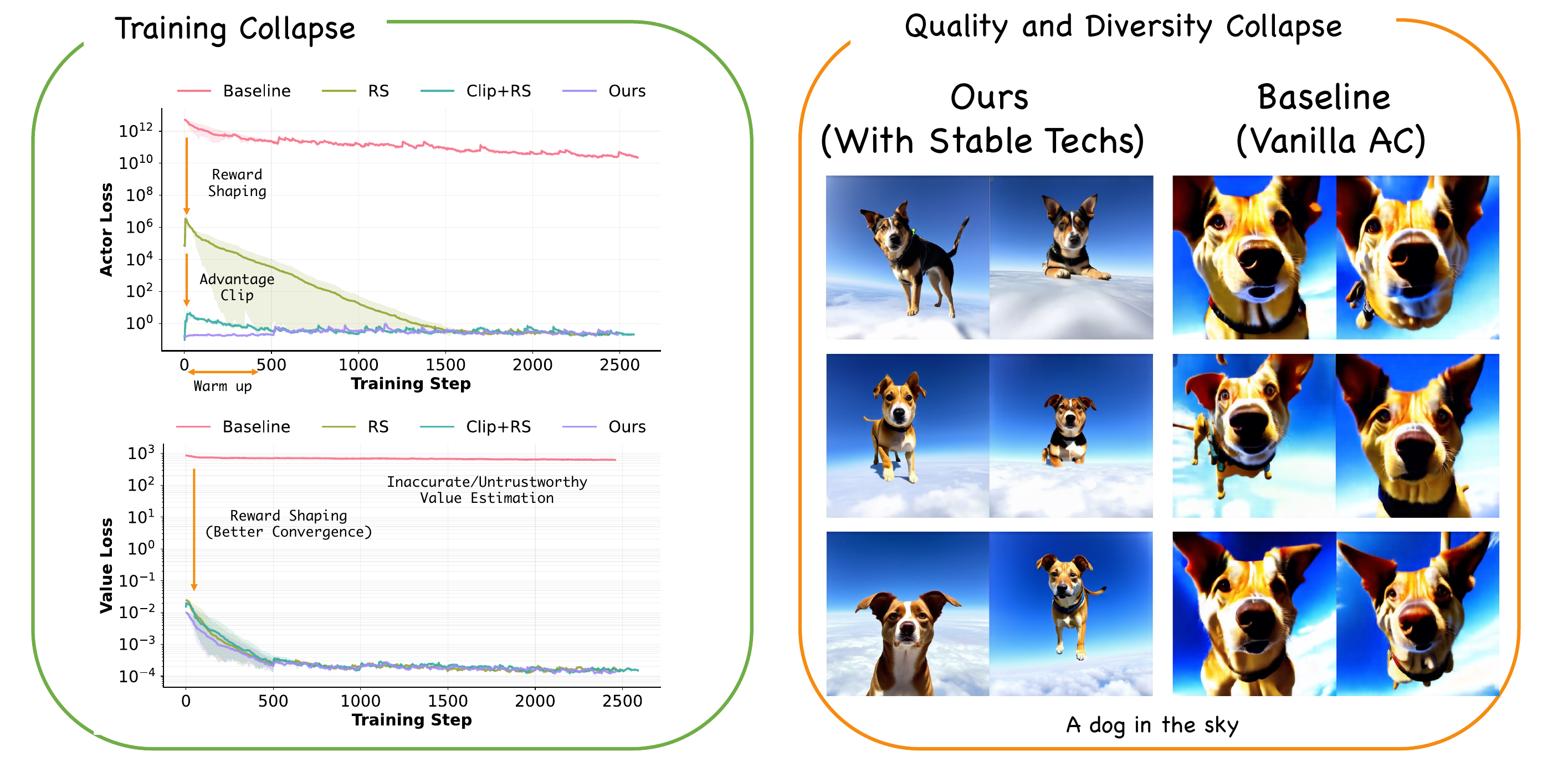}
	\caption{Ablation Studies of Proposed Stabilization Techniques.}
    \label{fig: exp stable critic training}
\end{figure}

\textbf{Training stability improvements.} Ablation studies in Fig. \ref{fig: exp stable critic training} reveal the effectiveness of our proposed components in stabilizing actor-critic training for flow matching models. The baseline method, lacking any stabilization techniques, exhibits severely unstable behavior in both critic and actor learning, as shown by the highly fluctuating loss curves (left). When incorporating reward shaping (RS), we observe a dramatic improvement in training stability, with the actor loss decreasing by multiple orders of magnitude from $10^{12}$ to approximately $10^2$. Adding advantage clipping (RS+Clip) further reduces fluctuations in both actor and value losses, preventing aggressive policy changes based on uncertain value estimates.

\textbf{Quality and diversity preservation.} Our full method, combining reward shaping, advantage clipping, and warm-up strategies, achieves superior stability across both loss metrics while producing substantially better outputs. The quality and diversity comparison (right) demonstrates this clearly: while the baseline (i.e., Vanilla AC, no stable techniques) suffers from severe visual artifacts and distortion in the "dog in the sky" generations, our stabilized approach maintains consistent visual coherence and reasonable diversity across samples. Most notably, our results demonstrate that reliable value estimation and stable policy improvement can be achieved even with relatively simple network architectures when using our stabilization techniques. The clear progression from baseline to our full method provides strong empirical support for the synergistic benefits of our three key components. See App. \ref{app: Additional Experimental Results} for learning curves.

\subsection{Qualitative Comparison in Complex Semantic Alignment Tasks}

\begin{figure*}[!t]
	\centering
 \includegraphics[width=0.8\linewidth]{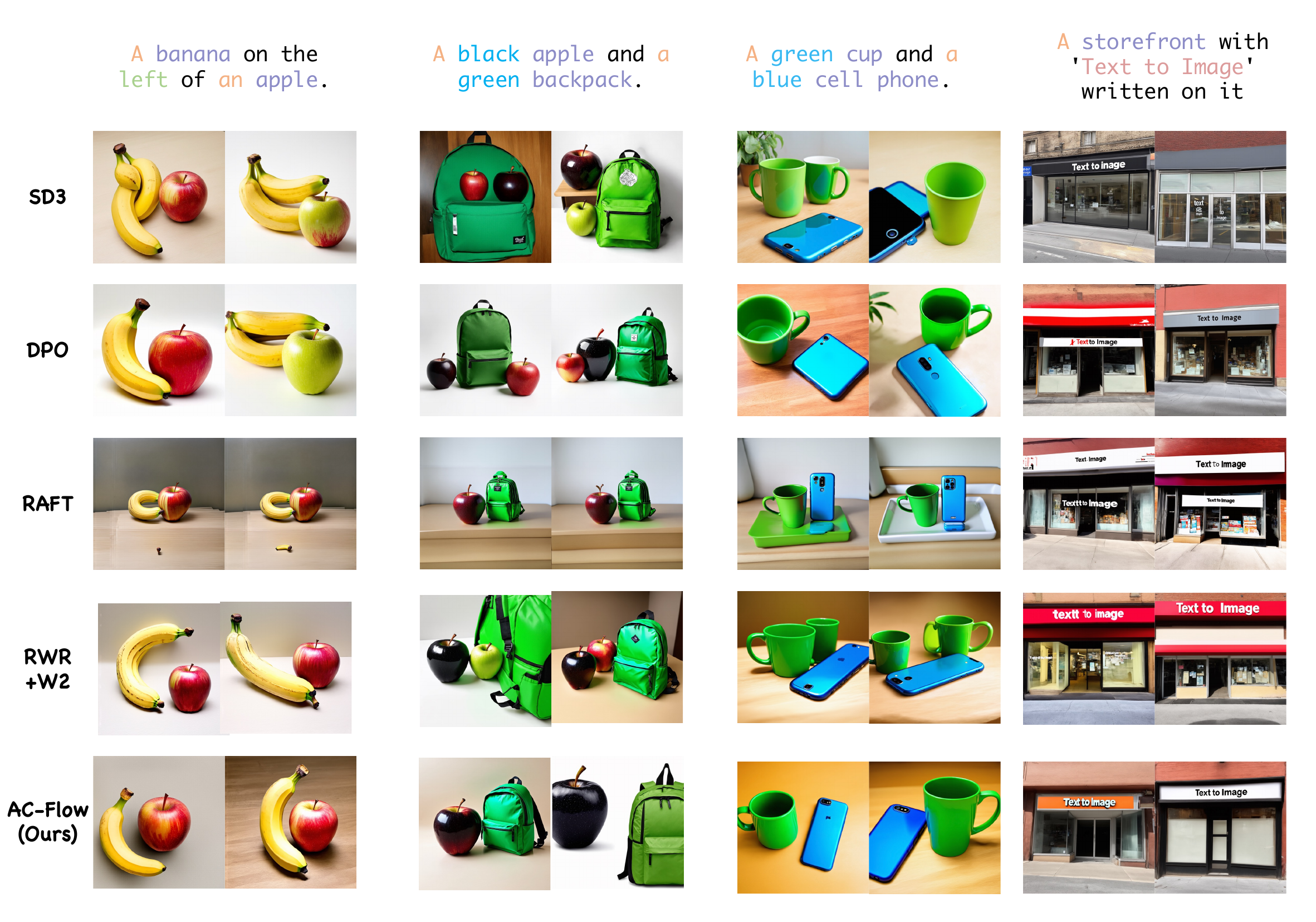}
	\caption{Qualitative Comparison. Our AC-Flow demonstrates superior performance across \textcolor[RGB]{177, 208, 149}{Spatial Positioning}, \textcolor[RGB]{79, 173, 234}{Coloring}, \textcolor[RGB]{139, 139, 195 }{Attribute Binding}, \textcolor[RGB]{233, 180, 138}{Object Counting}, and \textcolor[RGB]{210, 158, 157}{Text Rendering}. More generative results can be found in App. \ref{app: Additional Experimental Results}.}
    \label{fig: main com complex semantic alignment}
\end{figure*}

As illustrated in Figure \ref{fig: main com complex semantic alignment}, our AC-Flow demonstrates remarkable capabilities in handling challenging text-to-image alignment tasks compared to baseline models.

\textbf{Superior performance in spatial and attribute relationships.} For spatial relationship prompts like "a banana on the left of an apple," our approach precisely positions objects while maintaining natural visual quality, whereas SD3 and diffusion-DPO often struggle with accurate spatial arrangement while RAFT and RWR has to sacrifice generative quality and diversity. In attribute binding challenges exemplified by "a black apple and a green backpack," AC-Flow correctly assigns color attributes to their respective objects, while baseline methods frequently exhibit attribute confusion or semantic drift. Our method similarly excels at complex object relationships in "a green cup and a blue cell phone," where intermediate state evaluation provides granular control over the generation.

\textbf{Better text rendering and semantic control.} Perhaps most notably, AC-Flow produces significantly clearer text rendering in the "Text to Image" storefront prompt, a particularly challenging task for generative models. These qualitative improvements stem from our approach's unique ability to evaluate and optimize intermediate states during the generative trajectory, rather than relying solely on outcome rewards. The results align with our quantitative findings in Table \ref{tab:main_results}, confirming that our actor-critic framework with stable intermediate value estimation enables more precise semantic control while maintaining visual consistency and generative diversity. See App. \ref{app: Additional Experimental Results} for more results.

\section{Conclusion}
\label{sec: conclusion}

In this paper, we present AC-Flow, the first framework to successfully enable online RL fine-tuning of flow matching generative models with intermediate feedback. By leveraging advantage functions to provide granular credit assignment throughout the generative process, we overcome fundamental limitations of previous outcome-reward approaches (e.g., uniform credit assignment). Our comprehensive stabilization strategy combines reward shaping, advantage clipping, and warm-up phases to ensure stable training while preventing model collapse through Wasserstein regularization. Extensive experiments on SD3 demonstrate state-of-the-art performance in text-to-image alignment tasks and superior generalization to unseen human preference models (See Tab. \ref{tab:main_results}). The qualitative results show AC-Flow's ability to precisely control spatial relationships and attribute binding while maintaining visual quality (See Fig. \ref{fig: main com complex semantic alignment}). Our approach achieves these improvements with reasonable computational overhead, demonstrating that lightweight critic architectures can effectively fine-tune FM models without compromising between stability, efficiency, and generative quality—a significant advancement over previous approaches that struggled with the credit assignment challenge in FM models. More discussion on our limitations and broader impacts can be found in App. \ref{app: discussion}.

\bibliographystyle{plain}
\bibliography{example_paper}

\begin{thebibliography}{10}

\bibitem{collapse2}
Sina Alemohammad, Josue Casco{-}Rodriguez, Lorenzo Luzi, Ahmed~Imtiaz Humayun, Hossein Babaei, Daniel LeJeune, Ali Siahkoohi, and Richard~G. Baraniuk.
\newblock Self-consuming generative models go {MAD}.
\newblock In {\em The Twelfth International Conference on Learning Representations, {ICLR} 2024, Vienna, Austria, May 7-11, 2024}. OpenReview.net, 2024.

\bibitem{layernorm}
Lei~Jimmy Ba, Jamie~Ryan Kiros, and Geoffrey~E. Hinton.
\newblock Layer normalization.
\newblock {\em CoRR}, abs/1607.06450, 2016.

\bibitem{dp1957}
Richard Bellman.
\newblock {\em {Dynamic Programming}}.
\newblock Dover Publications, 1957.

\bibitem{ddpo}
Kevin Black, Michael Janner, Yilun Du, Ilya Kostrikov, and Sergey Levine.
\newblock Training diffusion models with reinforcement learning.
\newblock In {\em The Twelfth International Conference on Learning Representations, {ICLR} 2024, Vienna, Austria, May 7-11, 2024}. OpenReview.net, 2024.

\bibitem{Neuralode2018}
Tian~Qi Chen, Yulia Rubanova, Jesse Bettencourt, and David Duvenaud.
\newblock Neural ordinary differential equations.
\newblock In Samy Bengio, Hanna~M. Wallach, Hugo Larochelle, Kristen Grauman, Nicol{\`{o}} Cesa{-}Bianchi, and Roman Garnett, editors, {\em Advances in Neural Information Processing Systems 31: Annual Conference on Neural Information Processing Systems 2018, NeurIPS 2018, December 3-8, 2018, Montr{\'{e}}al, Canada}, pages 6572--6583, 2018.

\bibitem{diff_rewards}
Kevin Clark, Paul Vicol, Kevin Swersky, and David~J. Fleet.
\newblock Directly fine-tuning diffusion models on differentiable rewards.
\newblock In {\em The Twelfth International Conference on Learning Representations, {ICLR} 2024, Vienna, Austria, May 7-11, 2024}. OpenReview.net, 2024.

\bibitem{policy_collapse}
Shibhansh Dohare, Qingfeng Lan, and A~Rupam Mahmood.
\newblock Overcoming policy collapse in deep reinforcement learning.
\newblock In {\em Sixteenth European Workshop on Reinforcement Learning}, 2023.

\bibitem{collapse_llm1}
Elvis Dohmatob, Yunzhen Feng, Pu~Yang, Francois Charton, and Julia Kempe.
\newblock A tale of tails: Model collapse as a change of scaling laws.
\newblock {\em arXiv preprint arXiv:2402.07043}, 2024.

\bibitem{adjoint_matching}
Carles Domingo-Enrich, Michal Drozdzal, Brian Karrer, and Ricky T.~Q. Chen.
\newblock Adjoint matching: Fine-tuning flow and diffusion generative models with memoryless stochastic optimal control.
\newblock In {\em The Thirteenth International Conference on Learning Representations}, 2025.

\bibitem{Raft}
Hanze Dong, Wei Xiong, Deepanshu Goyal, Yihan Zhang, Winnie Chow, Rui Pan, Shizhe Diao, Jipeng Zhang, Kashun Shum, and Tong Zhang.
\newblock Raft: Reward ranked finetuning for generative foundation model alignment.
\newblock {\em arXiv preprint arXiv:2304.06767}, 2023.

\bibitem{sd3}
Patrick Esser, Sumith Kulal, Andreas Blattmann, Rahim Entezari, Jonas M{\"{u}}ller, Harry Saini, Yam Levi, Dominik Lorenz, Axel Sauer, Frederic Boesel, Dustin Podell, Tim Dockhorn, Zion English, and Robin Rombach.
\newblock Scaling rectified flow transformers for high-resolution image synthesis.
\newblock In {\em Forty-first International Conference on Machine Learning, {ICML} 2024, Vienna, Austria, July 21-27, 2024}. OpenReview.net, 2024.

\bibitem{atari_review}
Jiajun Fan.
\newblock A review for deep reinforcement learning in atari: Benchmarks, challenges, and solutions.
\newblock {\em CoRR}, abs/2112.04145, 2021.

\bibitem{pi2}
Jiajun Fan, He~Ba, Xian Guo, and Jianye Hao.
\newblock Critic {PI2:} master continuous planning via policy improvement with path integrals and deep actor-critic reinforcement learning.
\newblock {\em CoRR}, abs/2011.06752, 2020.

\bibitem{iclr_rwr}
Jiajun Fan, Shuaike Shen, Chaoran Cheng, Yuxin Chen, Chumeng Liang, and Ge~Liu.
\newblock Online reward-weighted fine-tuning of flow matching with wasserstein regularization.
\newblock In {\em The Thirteenth International Conference on Learning Representations, {ICLR} 2025, Singapore, April 24-28, 2025}. OpenReview.net, 2025.

\bibitem{gdi}
Jiajun Fan and Changnan Xiao.
\newblock Generalized data distribution iteration.
\newblock In Kamalika Chaudhuri, Stefanie Jegelka, Le~Song, Csaba Szepesv{\'{a}}ri, Gang Niu, and Sivan Sabato, editors, {\em International Conference on Machine Learning, {ICML} 2022, 17-23 July 2022, Baltimore, Maryland, {USA}}, volume 162 of {\em Proceedings of Machine Learning Research}, pages 6103--6184. {PMLR}, 2022.

\bibitem{lbc}
Jiajun Fan, Yuzheng Zhuang, Yuecheng Liu, Jianye Hao, Bin Wang, Jiangcheng Zhu, Hao Wang, and Shu{-}Tao Xia.
\newblock Learnable behavior control: Breaking atari human world records via sample-efficient behavior selection.
\newblock In {\em The Eleventh International Conference on Learning Representations, {ICLR} 2023, Kigali, Rwanda, May 1-5, 2023}. OpenReview.net, 2023.

\bibitem{dpok}
Ying Fan, Olivia Watkins, Yuqing Du, Hao Liu, Moonkyung Ryu, Craig Boutilier, Pieter Abbeel, Mohammad Ghavamzadeh, Kangwook Lee, and Kimin Lee.
\newblock {DPOK:} reinforcement learning for fine-tuning text-to-image diffusion models.
\newblock {\em CoRR}, abs/2305.16381, 2023.

\bibitem{Rest}
{\c{C}}aglar G{\"{u}}l{\c{c}}ehre, Tom~Le Paine, Srivatsan Srinivasan, Ksenia Konyushkova, Lotte Weerts, Abhishek Sharma, Aditya Siddhant, Alex Ahern, Miaosen Wang, Chenjie Gu, Wolfgang Macherey, Arnaud Doucet, Orhan Firat, and Nando de~Freitas.
\newblock Reinforced self-training (rest) for language modeling.
\newblock {\em CoRR}, abs/2308.08998, 2023.

\bibitem{guo2025deepseek_r1}
Daya Guo, Dejian Yang, Haowei Zhang, Junxiao Song, Ruoyu Zhang, Runxin Xu, Qihao Zhu, Shirong Ma, Peiyi Wang, Xiao Bi, et~al.
\newblock Deepseek-r1: Incentivizing reasoning capability in llms via reinforcement learning.
\newblock {\em arXiv preprint arXiv:2501.12948}, 2025.

\bibitem{collapse3}
Ari Holtzman, Jan Buys, Li~Du, Maxwell Forbes, and Yejin Choi.
\newblock The curious case of neural text degeneration.
\newblock In {\em 8th International Conference on Learning Representations, {ICLR} 2020, Addis Ababa, Ethiopia, April 26-30, 2020}. OpenReview.net, 2020.

\bibitem{lora}
Edward~J. Hu, Yelong Shen, Phillip Wallis, Zeyuan Allen{-}Zhu, Yuanzhi Li, Shean Wang, Lu~Wang, and Weizhu Chen.
\newblock Lora: Low-rank adaptation of large language models.
\newblock In {\em The Tenth International Conference on Learning Representations, {ICLR} 2022, Virtual Event, April 25-29, 2022}. OpenReview.net, 2022.

\bibitem{shape_rewards}
Yujing Hu, Weixun Wang, Hangtian Jia, Yixiang Wang, Yingfeng Chen, Jianye Hao, Feng Wu, and Changjie Fan.
\newblock Learning to utilize shaping rewards: {A} new approach of reward shaping.
\newblock In Hugo Larochelle, Marc'Aurelio Ranzato, Raia Hadsell, Maria{-}Florina Balcan, and Hsuan{-}Tien Lin, editors, {\em Advances in Neural Information Processing Systems 33: Annual Conference on Neural Information Processing Systems 2020, NeurIPS 2020, December 6-12, 2020, virtual}, 2020.

\bibitem{huang2020}
Chin{-}Wei Huang, Laurent Dinh, and Aaron~C. Courville.
\newblock Augmented normalizing flows: Bridging the gap between generative flows and latent variable models.
\newblock {\em CoRR}, abs/2002.07101, 2020.

\bibitem{se3}
Guillaume Huguet, James Vuckovic, Kilian Fatras, Eric Thibodeau-Laufer, Pablo Lemos, Riashat Islam, Cheng-Hao Liu, Jarrid Rector-Brooks, Tara Akhound-Sadegh, Michael Bronstein, et~al.
\newblock Sequence-augmented se (3)-flow matching for conditional protein backbone generation.
\newblock {\em Advances in neural information processing systems}, 2024.

\bibitem{ReFT}
Guillaume Huguet, James Vuckovic, Kilian Fatras, Eric Thibodeau{-}Laufer, Pablo Lemos, Riashat Islam, Cheng{-}Hao Liu, Jarrid Rector{-}Brooks, Tara Akhound{-}Sadegh, Michael~M. Bronstein, Alexander Tong, and Avishek~Joey Bose.
\newblock Sequence-augmented se(3)-flow matching for conditional protein backbone generation.
\newblock {\em CoRR}, abs/2405.20313, 2024.

\bibitem{r2d2}
Steven Kapturowski, Georg Ostrovski, John Quan, R{\'{e}}mi Munos, and Will Dabney.
\newblock Recurrent experience replay in distributed reinforcement learning.
\newblock In {\em 7th International Conference on Learning Representations, {ICLR} 2019, New Orleans, LA, USA, May 6-9, 2019}. OpenReview.net, 2019.

\bibitem{forget}
James Kirkpatrick, Razvan Pascanu, Neil~C. Rabinowitz, Joel Veness, Guillaume Desjardins, Andrei~A. Rusu, Kieran Milan, John Quan, Tiago Ramalho, Agnieszka Grabska{-}Barwinska, Demis Hassabis, Claudia Clopath, Dharshan Kumaran, and Raia Hadsell.
\newblock Overcoming catastrophic forgetting in neural networks.
\newblock {\em CoRR}, abs/1612.00796, 2016.

\bibitem{cnn_paper}
Alex Krizhevsky, Ilya Sutskever, and Geoffrey~E. Hinton.
\newblock Imagenet classification with deep convolutional neural networks.
\newblock {\em Commun. {ACM}}, 60(6):84--90, 2017.

\bibitem{anonymous2025derivativefree}
Xiner Li, Yulai Zhao, Chenyu Wang, Gabriele Scalia, G{\"{o}}kcen Eraslan, Surag Nair, Tommaso Biancalani, Aviv Regev, Sergey Levine, and Masatoshi Uehara.
\newblock Derivative-free guidance in continuous and discrete diffusion models with soft value-based decoding.
\newblock {\em CoRR}, abs/2408.08252, 2024.

\bibitem{sp_vla}
Ye~Li, Yuan Meng, Zewen Sun, Kangye Ji, Chen Tang, Jiajun Fan, Xinzhu Ma, Shutao Xia, Zhi Wang, and Wenwu Zhu.
\newblock {SP-VLA:} {A} joint model scheduling and token pruning approach for {VLA} model acceleration.
\newblock {\em CoRR}, abs/2506.12723, 2025.

\bibitem{prance}
Ye~Li, Chen Tang, Yuan Meng, Jiajun Fan, Zenghao Chai, Xinzhu Ma, Zhi Wang, and Wenwu Zhu.
\newblock {PRANCE:} joint token-optimization and structural channel-pruning for adaptive vit inference.
\newblock {\em CoRR}, abs/2407.05010, 2024.

\bibitem{ver_step_by_step}
Hunter Lightman, Vineet Kosaraju, Yuri Burda, Harrison Edwards, Bowen Baker, Teddy Lee, Jan Leike, John Schulman, Ilya Sutskever, and Karl Cobbe.
\newblock Let's verify step by step.
\newblock In {\em The Twelfth International Conference on Learning Representations, {ICLR} 2024, Vienna, Austria, May 7-11, 2024}. OpenReview.net, 2024.

\bibitem{fmgm}
Yaron Lipman, Ricky T.~Q. Chen, Heli Ben{-}Hamu, Maximilian Nickel, and Matthew Le.
\newblock Flow matching for generative modeling.
\newblock In {\em The Eleventh International Conference on Learning Representations, {ICLR} 2023, Kigali, Rwanda, May 1-5, 2023}. OpenReview.net, 2023.

\bibitem{instructed_rlhf}
Long Ouyang, Jeffrey Wu, Xu~Jiang, Diogo Almeida, Carroll~L. Wainwright, Pamela Mishkin, Chong Zhang, Sandhini Agarwal, Katarina Slama, Alex Ray, John Schulman, Jacob Hilton, Fraser Kelton, Luke Miller, Maddie Simens, Amanda Askell, Peter Welinder, Paul~F. Christiano, Jan Leike, and Ryan Lowe.
\newblock Training language models to follow instructions with human feedback.
\newblock In Sanmi Koyejo, S.~Mohamed, A.~Agarwal, Danielle Belgrave, K.~Cho, and A.~Oh, editors, {\em Advances in Neural Information Processing Systems 35: Annual Conference on Neural Information Processing Systems 2022, NeurIPS 2022, New Orleans, LA, USA, November 28 - December 9, 2022}, 2022.

\bibitem{pal1992multilayer}
Sankar~K Pal and Sushmita Mitra.
\newblock Multilayer perceptron, fuzzy sets, and classification.
\newblock {\em IEEE Transactions on neural networks}, 3(5):683--697, 1992.

\bibitem{flow_q}
Seohong Park, Qiyang Li, and Sergey Levine.
\newblock Flow q-learning.
\newblock {\em CoRR}, abs/2502.02538, 2025.

\bibitem{min_max_normal}
S.~Gopal~Krishna Patro and Kishore~Kumar Sahu.
\newblock Normalization: {A} preprocessing stage.
\newblock {\em CoRR}, abs/1503.06462, 2015.

\bibitem{awr}
Xue~Bin Peng, Aviral Kumar, Grace Zhang, and Sergey Levine.
\newblock Advantage-weighted regression: Simple and scalable off-policy reinforcement learning.
\newblock {\em CoRR}, abs/1910.00177, 2019.

\bibitem{rwr}
Jan Peters and Stefan Schaal.
\newblock Reinforcement learning by reward-weighted regression for operational space control.
\newblock In Zoubin Ghahramani, editor, {\em Machine Learning, Proceedings of the Twenty-Fourth International Conference {(ICML} 2007), Corvallis, Oregon, USA, June 20-24, 2007}, volume 227 of {\em {ACM} International Conference Proceeding Series}, pages 745--750. {ACM}, 2007.

\bibitem{clip}
Alec Radford, Jong~Wook Kim, Chris Hallacy, Aditya Ramesh, Gabriel Goh, Sandhini Agarwal, Girish Sastry, Amanda Askell, Pamela Mishkin, Jack Clark, Gretchen Krueger, and Ilya Sutskever.
\newblock Learning transferable visual models from natural language supervision.
\newblock In Marina Meila and Tong Zhang, editors, {\em Proceedings of the 38th International Conference on Machine Learning, {ICML} 2021, 18-24 July 2021, Virtual Event}, volume 139 of {\em Proceedings of Machine Learning Research}, pages 8748--8763. {PMLR}, 2021.

\bibitem{dpo}
Rafael Rafailov, Archit Sharma, Eric Mitchell, Christopher~D. Manning, Stefano Ermon, and Chelsea Finn.
\newblock Direct preference optimization: Your language model is secretly a reward model.
\newblock In Alice Oh, Tristan Naumann, Amir Globerson, Kate Saenko, Moritz Hardt, and Sergey Levine, editors, {\em Advances in Neural Information Processing Systems 36: Annual Conference on Neural Information Processing Systems 2023, NeurIPS 2023, New Orleans, LA, USA, December 10 - 16, 2023}, 2023.

\bibitem{draw_bench}
Chitwan Saharia, William Chan, Saurabh Saxena, Lala Li, Jay Whang, Emily~L. Denton, Seyed Kamyar~Seyed Ghasemipour, Raphael~Gontijo Lopes, Burcu~Karagol Ayan, Tim Salimans, Jonathan Ho, David~J. Fleet, and Mohammad Norouzi.
\newblock Photorealistic text-to-image diffusion models with deep language understanding.
\newblock In Sanmi Koyejo, S.~Mohamed, A.~Agarwal, Danielle Belgrave, K.~Cho, and A.~Oh, editors, {\em Advances in Neural Information Processing Systems 35: Annual Conference on Neural Information Processing Systems 2022, NeurIPS 2022, New Orleans, LA, USA, November 28 - December 9, 2022}, 2022.

\bibitem{gae}
John Schulman, Philipp Moritz, Sergey Levine, Michael~I. Jordan, and Pieter Abbeel.
\newblock High-dimensional continuous control using generalized advantage estimation.
\newblock In Yoshua Bengio and Yann LeCun, editors, {\em 4th International Conference on Learning Representations, {ICLR} 2016, San Juan, Puerto Rico, May 2-4, 2016, Conference Track Proceedings}, 2016.

\bibitem{grpo}
Zhihong Shao, Peiyi Wang, Qihao Zhu, Runxin Xu, Junxiao Song, Mingchuan Zhang, Y.~K. Li, Y.~Wu, and Daya Guo.
\newblock Deepseekmath: Pushing the limits of mathematical reasoning in open language models.
\newblock {\em CoRR}, abs/2402.03300, 2024.

\bibitem{model_collapse}
Ilia Shumailov, Zakhar Shumaylov, Yiren Zhao, Nicolas Papernot, Ross Anderson, and Yarin Gal.
\newblock Ai models collapse when trained on recursively generated data.
\newblock {\em Nature}, 631(8022):755--759, 2024.

\bibitem{collapse_llm2}
Yi~Su, Yixin Ji, Juntao Li, Hai Ye, and Min Zhang.
\newblock Beware of model collapse! fast and stable test-time adaptation for robust question answering.
\newblock In {\em Proceedings of the 2023 Conference on Empirical Methods in Natural Language Processing}, pages 12998--13011, 2023.

\bibitem{rlsutton}
Richard~S. Sutton and Andrew~G. Barto.
\newblock Reinforcement learning: An introduction.
\newblock {\em {IEEE} Trans. Neural Networks}, 9(5):1054--1054, 1998.

\bibitem{otcfm}
Alexander Tong, Kilian Fatras, Nikolay Malkin, Guillaume Huguet, Yanlei Zhang, Jarrid Rector{-}Brooks, Guy Wolf, and Yoshua Bengio.
\newblock Improving and generalizing flow-based generative models with minibatch optimal transport.
\newblock {\em Trans. Mach. Learn. Res.}, 2024, 2024.

\bibitem{Complexity_of_Compute}
Leslie~G. Valiant.
\newblock The complexity of computing the permanent.
\newblock {\em Theor. Comput. Sci.}, 8:189--201, 1979.

\bibitem{diff_dpo}
Bram Wallace, Meihua Dang, Rafael Rafailov, Linqi Zhou, Aaron Lou, Senthil Purushwalkam, Stefano Ermon, Caiming Xiong, Shafiq Joty, and Nikhil Naik.
\newblock Diffusion model alignment using direct preference optimization.
\newblock In {\em {IEEE/CVF} Conference on Computer Vision and Pattern Recognition, {CVPR} 2024, Seattle, WA, USA, June 16-22, 2024}, pages 8228--8238. {IEEE}, 2024.

\bibitem{escl}
Hao Wang, Zhichao Chen, Jiajun Fan, Yuxin Huang, Weiming Liu, and Xinggao Liu.
\newblock Entire space counterfactual learning: Tuning, analytical properties and industrial applications.
\newblock {\em CoRR}, abs/2210.11039, 2022.

\bibitem{ot_tee_nips23}
Hao Wang, Jiajun Fan, Zhichao Chen, Haoxuan Li, Weiming Liu, Tianqiao Liu, Quanyu Dai, Yichao Wang, Zhenhua Dong, and Ruiming Tang.
\newblock Optimal transport for treatment effect estimation.
\newblock In Alice Oh, Tristan Naumann, Amir Globerson, Kate Saenko, Moritz Hardt, and Sergey Levine, editors, {\em Advances in Neural Information Processing Systems 36: Annual Conference on Neural Information Processing Systems 2023, NeurIPS 2023, New Orleans, LA, USA, December 10 - 16, 2023}, 2023.

\bibitem{con_former}
Hao Wang, Jianxun Lian, Mingqi Wu, Haoxuan Li, Jiajun Fan, Wanyue Xu, Chaozhuo Li, and Xing Xie.
\newblock Convformer: Revisiting transformer for sequential user modeling.
\newblock {\em CoRR}, abs/2308.02925, 2023.

\bibitem{pzero}
Ziwen Wang, Jiajun Fan, Ruihan Guo, Thao Nguyen, Heng Ji, and Ge~Liu.
\newblock Proteinzero: Self-improving protein generation via online reinforcement learning.
\newblock {\em CoRR}, abs/2506.07459, 2025.

\bibitem{varcon}
Ziwen Wang, Jiajun Fan, Thao Nguyen, Heng Ji, and Ge~Liu.
\newblock Variational supervised contrastive learning.
\newblock {\em CoRR}, abs/2506.07413, 2025.

\bibitem{hpsv2}
Xiaoshi Wu, Yiming Hao, Keqiang Sun, Yixiong Chen, Feng Zhu, Rui Zhao, and Hongsheng Li.
\newblock Human preference score v2: {A} solid benchmark for evaluating human preferences of text-to-image synthesis.
\newblock {\em CoRR}, abs/2306.09341, 2023.

\bibitem{casa}
Changnan Xiao, Haosen Shi, Jiajun Fan, and Shihong Deng.
\newblock {CASA:} {A} bridge between gradient of policy improvement and policy evaluation.
\newblock {\em CoRR}, abs/2105.03923, 2021.

\bibitem{cn_rl_entropy}
Changnan Xiao, Haosen Shi, Jiajun Fan, and Shihong Deng.
\newblock An entropy regularization free mechanism for policy-based reinforcement learning.
\newblock {\em CoRR}, abs/2106.00707, 2021.

\bibitem{image_reward}
Jiazheng Xu, Xiao Liu, Yuchen Wu, Yuxuan Tong, Qinkai Li, Ming Ding, Jie Tang, and Yuxiao Dong.
\newblock Imagereward: Learning and evaluating human preferences for text-to-image generation.
\newblock In Alice Oh, Tristan Naumann, Amir Globerson, Kate Saenko, Moritz Hardt, and Sergey Levine, editors, {\em Advances in Neural Information Processing Systems 36: Annual Conference on Neural Information Processing Systems 2023, NeurIPS 2023, New Orleans, LA, USA, December 10 - 16, 2023}, 2023.

\bibitem{zhang2024restmcts}
Dan Zhang, Sining Zhoubian, Ziniu Hu, Yisong Yue, Yuxiao Dong, and Jie Tang.
\newblock Re{ST}-{MCTS}*: {LLM} self-training via process reward guided tree search.
\newblock In {\em The Thirty-eighth Annual Conference on Neural Information Processing Systems}, 2024.

\bibitem{xiangxin_stable}
Xiangxin Zhou, Liang Wang, and Yichi Zhou.
\newblock Stabilizing policy gradients for stochastic differential equations via consistency with perturbation process.
\newblock In {\em Forty-first International Conference on Machine Learning, {ICML} 2024, Vienna, Austria, July 21-27, 2024}. OpenReview.net, 2024.

\bibitem{offline_to_online}
Zhiyuan Zhou, Andy Peng, Qiyang Li, Sergey Levine, and Aviral Kumar.
\newblock Efficient online reinforcement learning fine-tuning need not retain offline data.
\newblock {\em CoRR}, abs/2412.07762, 2024.

\end{thebibliography}


\clearpage

\appendix


\section{Discussion}
\label{app: discussion}

\subsection{Motivation of Online RL Fine-tuning}

While traditional approaches like SFT \citep{instructed_rlhf} and offline RL \citep{dpo} have demonstrated success in training large generative models like SD3 \citep{sd3}, recent breakthroughs - particularly DeepSeek-R1's impressive results with online RL for LLMs \citep{guo2025deepseek_r1} - have revealed the immense potential of online reinforcement learning (even with simple GRPO \citep{grpo}). The ability to stably and continuously improve model performance using self-generated data \citep{pzero}, without requiring extensive human-collected datasets, represents a significant advancement in the field. However, successfully implementing online fine-tuning for large-scale flow matching models presents unique challenges that our work systematically addresses.

\subsection{Core Technical Challenges and Solutions}

The primary challenge in online fine-tuning  of flow matching models from process feedback stems from the  instability of actor-critic training in continuous-time and online settings. While GRPO \citep{grpo}, DeepSeek-R1 \citep{guo2025deepseek_r1}  and similar approaches achieve stability by avoiding critic estimation entirely, this comes at the cost of ignoring credit assignment for intermediate states. Our work demonstrates that it is possible to maintain training stability while leveraging the benefits of critic-based intermediate state evaluation through careful design choices:

{\begin{mdframed}[backgroundcolor=mygray]
\begin{enumerate}
     \item The introduction of a computationally efficient critic architecture proves that \textit{complex, computation-heavy networks aren't necessary for effective value estimation of intermediate states.}
    \item Our \textit{reward shaping and advantage clipping mechanisms prevent the destructive feedback loops} \citep{offline_to_online} that typically plague online actor-critic training while enabling stably improvement of model performance.
    \item The warm-up phase enables reliable critic learning before influencing policy updates, \textit{solving the cold-start problem.}
\end{enumerate}
\end{mdframed}}

\subsection{Robust and Easy-to-Use  Online RL Fine-tuning Framework}

Although RLHF \citep{instructed_rlhf,dpo,dpok,Raft,ReFT,Rest} has been extensively studied, existing methods often suffer from excessive complexity, computational inefficiency, numerous hyperparameters, and deployment challenges. Simple outcome-driven approaches \citep{Raft,ReFT,ddpo,iclr_rwr}, while easier to implement, fail to address the fundamental credit assignment problem. AC-Flow bridges this gap by providing a robust, easy-to-use framework that enables efficient fine-tuning of continuous-time flow matching models using purely self-generated data.

Previous attempts at training with self-generated data have struggled with instability, catastrophic forgetting, and model collapse \citep{model_collapse,ddpo,collapse_llm1,collapse_llm2,collapse2,policy_collapse,collapse3}. Our comprehensive stabilization techniques and Wasserstein regularization method enable stable value function learning using a computationally efficient critic architecture. The training curves demonstrate rapid convergence of both critic and actor losses to reasonable ranges while maintaining consistent improvement.

\subsection{Effective Credit Assignment without Excessive Computation Overhead}

Traditional approaches that apply outcome rewards across all intermediate states suffer from high variance and potentially misleading update signals \citep{rwr,iclr_rwr,ddpo,Raft,ReFT}. AC-Flow's critic-based advantage estimation provides granular feedback for each state in the generative trajectory. This precise credit assignment is achieved without too much computational overhead typically associated with critic networks, thanks to our efficient architecture design and stabilization techniques.

\subsection{Broader Impacts}
\label{app: Broader Impacts}
Our AC-Flow framework offers significant positive societal impacts through enabling more precise control over generative AI models with intermediate feedback, potentially leading to more aligned and reliable models. The computationally efficient critic architecture and stabilization techniques democratize access to high-quality fine-tuning with reduced computational requirements. Our Wasserstein regularization mechanism provides a built-in mitigation strategy by constraining model outputs while allowing for improvement, and the advantage clipping mechanism offers fine-grained control that could be adapted for safety-focused refinement. In fact, our method can help fine-tune the model to become more responsible and safer by adding some safety based reward. We recommend that implementations include appropriate content filtering systems and follow responsible AI release practices to minimize potential harms.

\subsection{Future Implications}

The success of AC-Flow in stabilizing online actor-critic training for flow matching models has broader implications for the field of generative AI. Our framework demonstrates that \textbf{the benefits of online RL - continuous improvement, data efficiency, and autonomous learning - can be realized without sacrificing training stability or computational efficiency}. This opens new possibilities for developing self-improving generative models that can adapt and enhance their capabilities through interaction with their own outputs.

The ability to achieve stable convergence and policy round through computationally efficient architectures and principled regularization suggests that online RL could become a more practical and widely-adopted approach for fine-tuning large-scale generative models. AC-Flow provides a foundation for future research into more efficient and robust online learning methods that balance performance improvements with computational constraints.

\clearpage
\section{Background}
\label{app: background}

In this section, we present the key concepts and frameworks that form the foundation of our approach. We begin by examining flow matching for conditional generation, followed by reinforcement learning methods for model alignment, and conclude with relevant policy optimization techniques.

\subsection{Flow Matching for Conditional Generation}

\subsubsection{Flow Matching Formulation}

Flow Matching \citep[FM]{fmgm} trains a time-dependent vector field $v_\theta(t, x)$ to transport samples from a base distribution $p_0\left(x_0\right)$ to a target $q\left(x_1\right)$ via the ODE:

\begin{equation}
    \frac{d x_t}{d t}=v_\theta\left(t, x_t\right), \quad x_0 \sim p_0\left(x_0\right)
\end{equation}

where $x_t$ denotes the state at time $t \in[0,1]$. The FM objective aligns $v_\theta$ with a target vector field $u_t(x)$, derived from the marginal probability path $p_t(x)$ :

\begin{equation}
    \mathcal{L}_{\mathrm{FM}}(\theta)=\mathbb{E}_{t \sim \mathcal{U}[0,1], x \sim p_t(x)}\left\|v_\theta(t, x)-u_t(x)\right\|^2 .
\end{equation}

Here, $u_t(x)$ satisfies the continuity equation:

\begin{equation}
    \frac{\partial p_t(x)}{\partial t}+\nabla_x \cdot\left(p_t(x) u_t(x)\right)=0
\end{equation}

ensuring $p_t(x)$ evolves from $p_0\left(x_0\right)$ to $q\left(x_1\right)$. However, computing $u_t(x)$ requires integrating over $q\left(x_1\right)$, which is intractable for high-dimensional tasks.

\subsubsection{Conditional Flow Matching (CFM)}

To resolve this, CFM \citep{otcfm} conditions on individual samples $x_1 \sim \boldsymbol{q}\left(x_1\right)$, defining a conditional probability path $p_t\left(x \mid x_1\right)$ and a per-sample vector field $u_t\left(x \mid x_1\right)$. The CFM objective becomes:

\begin{equation}
    \mathcal{L}_{\mathrm{CFM}}(\theta)=\mathbb{E}_{t \sim \mathcal{U}[0,1], x_1 \sim q\left(x_1\right), x \sim p_t\left(x \mid x_1\right)}\left\|v_\theta(t, x)-u_t\left(x \mid x_1\right)\right\|^2 .
\end{equation}

Key Properties:
\begin{itemize}
    \item \textbf{Tractability:} $u_t\left(x \mid x_1\right)$ is defined per-sample (e.g., $u_t\left(x \mid x_1\right)=x_1-x_0$ for linear paths).
    \item \textbf{Gradient Equivalence:} $\nabla_\theta \mathcal{L}_{\mathrm{FM}}(\theta)=\nabla_\theta \mathcal{L}_{\mathrm{CFM}}(\theta)$, making CFM a practical alternative.
\end{itemize}

\subsubsection{Flow Matching for Conditional Generation}

For conditional tasks (e.g., text-to-image synthesis), CFM incorporates auxiliary context c (e.g., text prompts). The target distribution becomes $q\left(x_1 \mid c\right)$, and the vector field $v_\theta(t, x, c)$ is conditioned on $c$:

\begin{equation}
\mathcal{L}_{\mathrm{CFM}-\operatorname{cond}}(\theta)=\mathbb{E}_{\substack{t \sim \mathcal{U}[0,1], c \sim p(c), x_1 \sim q\left(x_1 \mid c\right), x \sim p\left(x_t \mid x_1, c\right)}}\left\|v_\theta(t, x, c)-u_t\left(x \mid x_1, c\right)\right\|^2
\end{equation}

Conditional Probability Paths:

\textbf{Linear Interpolation:}

\begin{equation}
    x_t=(1-t) x_0+t x_1, \quad u_t\left(x \mid x_1, c\right)=x_1-x_0
\end{equation}

\textbf{Optimal Transport (OT) \citep{otcfm,ot_tee_nips23,con_former}:}

\begin{equation}
    x_t=x_0+t\left(x_1-x_0\right), \quad u_t\left(x \mid x_1, c\right)=x_1-x_0
\end{equation}

\textbf{Gaussian Paths:}

\begin{equation}
    p\left(x_t \mid x_1, c\right)=\mathcal{N}\left(x \mid \mu_t\left(x_1, c\right), \sigma_t^2\left(x_1, c\right) \mathbf{I}\right)
\end{equation}

with $u_t\left(x \mid x_1, c\right)=\frac{d \mu_t}{d t}+\left(\frac{d \sigma_t}{d t}\right) \sigma_t^{-1}\left(x-\mu_t\right)$.

\subsubsection{Push Forward  Mechanism}

The learned $v_\theta(t, x, c)$ induces a deterministic push-forward map $\Phi_{0 \rightarrow 1}^c$, transporting $p_0\left(x_0\right)$ to $q\left(x_1 \mid c\right):$

\begin{equation}
q\left(x_1 \mid c\right)=\Phi_{0 \rightarrow 1}^c \# p_0\left(x_0\right)
\end{equation}

where \# denotes the push-forward operation. This avoids explicit modeling of $q\left(x_1 \mid c\right)$, instead relying on the ODE's geometric transformation.

\subsubsection{Sampling Process}

To generate samples conditioned on $c$ :
\begin{enumerate}
    \item Draw $x_0 \sim p_0\left(x_0\right)$ and $c \sim p(c)$.
    \item Solve the ODE:
\begin{equation}
x_1=x_0+\int_0^1 v_\theta\left(t, x_t, c\right) d t
\end{equation}
Numerically, this is approximated using methods like Euler:

\begin{equation}
    x_{t_{k+1}}=x_{t_k}+\Delta t \cdot v_\theta\left(t_k, x_{t_k}, c\right), \quad t_k=k \Delta t
\end{equation}
\end{enumerate}

\subsubsection{Likelihood Calculation in Flow Matching}

In flow matching, the log-likelihood of a sample $x_1 \sim q\left(x_1\right)$ under the learned model $p_\theta\left(x_1\right)$ can be computed using the instantaneous change of variables formula \citep{Neuralode2018}:

\begin{equation}
\label{equ: likelihood in fm}
    \log p_\theta\left(x_1\right)=\log p_0\left(x_0\right)-\int_0^1 \nabla_x \cdot v_\theta\left(t, x_t\right) d t
\end{equation}

where $x_t$ follows the ODE $\frac{d x_t}{d t}=v_\theta\left(t, x_t\right)$, and $\nabla_x \cdot v_\theta$ is the divergence of the vector field.

\paragraph{Intractability.} \textbf{1. Divergence Computation:} Calculating $\nabla_x \cdot v_\theta\left(t, x_t\right)$ requires $\mathcal{O}\left(d^2\right)$ operations for $x_t \in \mathbb{R}^d$, which is prohibitive for high-dimensional data (e.g., images with $d \sim 10^6$ ).
\textbf{2. Numerical Integration:} Approximating $\int_0^1 \nabla_x \cdot v_\theta d t$ introduces cumulative errors, especially with adaptive step sizes.

\subsubsection{KL Divergence in Flow Matching}

The KL divergence $D_{\mathrm{KL}}\left(q\left(x_1\right) \| p_\theta\left(x_1\right)\right)$ measures the discrepancy between the target $q\left(x_1\right)$ and the model $p_\theta\left(x_1\right)$ :

\begin{equation}
    D_{\mathrm{KL}}\left(q \| p_\theta\right)=\mathbb{E}_{x_1 \sim q\left(x_1\right)}\left[\log q\left(x_1\right)-\log p_\theta\left(x_1\right)\right] .
\end{equation}

Using Equ. \eqref{equ: likelihood in fm}, this becomes:

\begin{equation}
\label{equ: kl in fm}
    D_{\mathrm{KL}}\left(q \| p_\theta\right)=\mathbb{E}_{q\left(x_1\right)}\left[\log q\left(x_1\right)-\log p_0\left(x_0\right)+\int_0^1 \nabla_x \cdot v_\theta\left(t, x_t\right) d t\right] .
\end{equation}

\paragraph{Intractability.} \textbf{1. High-Dimensional Expectation: }The expectation $\mathbb{E}_{q\left(x_1\right)}[\cdot]$ requires integration over $\mathbb{R}^d$, which is infeasible for large $d$. \textbf{2. Density Estimation:} $\log q\left(x_1\right)$ is often unknown (e.g., $q\left(x_1\right)$ is an implicit distribution).

\begin{theorem}[Intractability of Exact Likelihood]
\label{theorem: Intractability of Exact Likelihood}
    Under standard complexity-theoretic assumptions, computing $\log p_\theta\left(x_1\right)$ or $D_{\mathrm{KL}}\left(q \| p_\theta\right)$ for a flow matching model $v_\theta(t, x)$ is \#P-hard in the dimension $d$.
\end{theorem}

\begin{proof}[Proof of Theorem \ref{theorem: Intractability of Exact Likelihood}]
    \textbf{Reduction to Matrix Permanent:} Following the reduction in \citep{huang2020}, exact likelihood computation in continuous normalizing flows can be shown equivalent to computing matrix permanents, which is \#P-hard \citep{Complexity_of_Compute}.

    \textbf{Divergence as High-Dimensional Integral:} Equ. \eqref{equ: kl in fm} involves integrating over $\mathbb{R}^d$, which is known to suffer from the curse of dimensionality \citep{dp1957}.
\end{proof}




\subsection{RL Formulation for Generative Model Fine-tuning}

\subsubsection{RL Formulation}
Reinforcement learning (RL) frames generative model fine-tuning (e.g., flow matching, diffusion models, LLMs) as a sequential decision process. Let the policy $\pi_\theta$ model a trajectory from an initial state $x_0 \sim$ $p_0\left(x_0\right)$ to a final sample $x_T \sim \boldsymbol{q}\left(x_T \mid c\right)$, conditioned on context $c$ (e.g., text prompts). Formally, this is a continuous-time MDP:
\begin{enumerate}
    \item   State: $s_t=\left(x_t, t, c\right)$, where $x_t$ is the intermediate state at time $t \in[0,1]$.
    \item Action: $a_t=v_\theta\left(t, x_t, c\right)$, the policy's update direction (e.g., vector field in flow matching).
    \item  Reward: A predefined function $r\left(x_t, a_t, c\right)$ quantifying alignment with $c$.
\end{enumerate}

The objective is to maximize the expected cumulative reward:

\begin{equation}
    J\left(\pi_\theta\right)=\mathbb{E}_{\tau \sim \pi_\theta}\left[\int_0^1 r\left(x_t, a_t, c\right) d t\right] \overset{Sparse\ Rewards}{=} \mathbb{E}_{x_1 \sim \pi_\theta}\left[r\left(x_1,  c\right)\right] 
\end{equation}

where $\tau=\left\{s_t, a_t\right\}_{t=0}^1$ is a trajectory.

\begin{theorem}[Policy Gradient Theorem]
\begin{equation}
    \nabla_\theta J\left(\pi_\theta\right)=\mathbb{E}_{\tau \sim \pi_\theta}\left[\int_0^1 \nabla_\theta \log \pi_\theta\left(a_t \mid x_t, t, c\right) \cdot A^\pi\left(x_t, t, c, a_t\right) d t\right]
\end{equation}

where $A^\pi\left(x_t, t, c, a_t\right)=Q^\pi\left(x_t, t, c, a_t\right)-V^\pi\left(x_t, t, c\right)$ is the advantage function.
\end{theorem}

\subsubsection{Reinforcement Learning from Human Feedback (RLHF)}

RLHF incorporates human preferences via KL regularization to ensure the policy $\pi_\theta$ does not deviate excessively from a reference policy $\pi_{\text {ref }}$. The objective becomes:

\begin{equation}
    J_{\mathrm{RLHF}}\left(\pi_\theta\right)=\mathbb{E}_{\tau \sim \pi_\theta}\left[\int_0^1 r\left(x_t, a_t, c\right) d t\right]-\beta \mathbb{E}_{\tau \sim \pi_\theta}\left[\int_0^1 \mathrm{KL}\left(\pi_\theta\left(\cdot \mid x_t, t, c\right) \| \pi_{\mathrm{ref}}\left(\cdot \mid x_t, t, c\right)\right) d t\right]
\end{equation}

where $\beta>0$ controls regularization strength.
Theorem 1 (KL-Regularized Optimal Policy).
The optimal policy $\pi^*$ under $J_{\text {RLHF }}$ satisfies:

\begin{equation}
    \pi^*\left(a_t \mid x_t, t, c\right) \propto \pi_{\mathrm{ref}}\left(a_t \mid x_t, t, c\right) \exp \left(\frac{1}{\beta} A^\pi\left(x_t, t, c, a_t\right)\right) .
\end{equation}

\begin{remark}
    Maximize $J_{\mathrm{RLHF}}\left(\pi_\theta\right)$ using variational calculus, yielding the exponentiated advantage form.
\end{remark}

\subsubsection{Value Function}

Value functions evaluate states/actions explicitly conditioned on $c$ \citep{casa,atari_review,escl,pi2}:

\textbf{State Value:}

\begin{equation}
    V^\pi\left(x_t, t, c\right)=\mathbb{E}_{\pi_\theta}\left[\int_t^1 r\left(x_\tau, a_\tau, c\right) d \tau \mid x_t, t, c\right] .
\end{equation}

\textbf{Action Value:}

\begin{equation}
    Q^\pi\left(x_t, t, c, a_t\right)=r\left(x_t, a_t, c\right)+\mathbb{E}_{\pi_\theta}\left[V^\pi\left(x_{t+\Delta t}, t+\Delta t, c\right) \mid x_t, t, c, a_t\right]
\end{equation}

\textbf{Advantage:}

\begin{equation}
    A^\pi\left(x_t, t, c, a_t\right)=Q^\pi\left(x_t, t, c, a_t\right)-V^\pi\left(x_t, t, c\right) .
\end{equation}

\textbf{Bellman Equation for Continuous Time:}

\begin{equation}
    \frac{\partial V^\pi\left(x_t, t, c\right)}{\partial t}+\max _{a_t}\left(r\left(x_t, a_t, c\right)+\nabla_{x_t} V^\pi\left(x_t, t, c\right) \cdot f\left(x_t, a_t, c\right)\right)=0
\end{equation}

where $f\left(x_t, a_t, c\right)=\frac{d x_t}{d t}$ defines system dynamics (e.g., flow matching ODE).

\subsubsection{Value Estimation via Reward Regression}

For sparse rewards, value regression directly estimates $V_\psi\left(x_t, t, c\right)$ from observed rewards:
Objective \citep{gdi,lbc,atari_review}:

\begin{equation}
    \min _\psi \mathbb{E}_{\mathcal{D}}\left[\left(V_\psi\left(x_t, t, c\right)-\int_t^1 r\left(x_\tau, a_\tau, c\right) d \tau\right)^2\right]
\end{equation}

where $\mathcal{D}$ contains trajectories with rewards conditioned on $c$.

For simplicity, this article only uses the reward model in the last step and obtains the value of the intermediate moments by learning the value function (critic model). Therefore the critic loss of Reward Regression method can be re-write as:
\begin{equation}
    \min _\psi \mathbb{E}_{\mathcal{D}}\left[\left(V_\psi\left(x_t, t, c\right)-r(x_1,c)\right)^2\right],
\end{equation}
which is tractable.

\begin{theorem}[Value Regression Consistency]
    If $V_\psi\left(x_t, t, c\right)$ approximates $V^\pi\left(x_t, t, c\right)$ with error $\epsilon$, the policy gradient error is bounded by $C \epsilon$ for some $C>0$.
\end{theorem}

\begin{definition}[Admissible Value Function]
A value function $V_\psi\left(x_t, t, c\right)$ is admissible if:

\begin{equation}
\left|V_\psi\left(x_t, t, c\right)-V^\pi\left(x_t, t, c\right)\right| \leq \epsilon \quad \forall\left(x_t, t, c\right) .
\end{equation}
\end{definition}


\subsection{Reward-Based RL Fine-tuning  Methods}

\subsubsection{Reward-Weighted Regression (RWR)}

Reward-Weighted Regression (RWR) is a foundational method for fine-tuning generative models by reweighting trajectories based on their final rewards. Given a reward function $r\left(x_1, c\right)$ evaluating terminal states $x_1 \sim p_\theta\left(x_1 \mid c\right)$, RWR assigns weights to trajectories using:

\begin{equation}
    w\left(x_1, c\right) \propto \exp \left(\tau r\left(x_1, c\right)\right), \quad \tau>0
\end{equation}

where $\tau$ controls reward sensitivity. The policy update minimizes a weighted alignment loss:

\begin{equation}
    \mathcal{L}_{\mathrm{RWR}}(\theta)=\mathbb{E}_{t, x_1, c, x}\left[w\left(x_1, c\right)\left\|v_\theta(t, x, c)-u_t\left(x \mid x_1, c\right)\right\|^2\right]
\end{equation}

with $t \sim \mathcal{U}[0,1], x_1 \sim p_\theta^{(n)}, c \sim p(c)$, and $x \sim p\left(x_t \mid x_1, c\right)$.

\begin{theorem}[RWR Policy Update]
    Under ideal conditions, iteratively minimizing $\mathcal{L}_{\mathrm{RWR}}$ induces a policy update:
    \begin{equation}
        p_\theta^{(n+1)}(x \mid c) \propto \exp (\tau r(x, c)) p_\theta^{(n)}(x \mid c)
    \end{equation}
\end{theorem}

\begin{remark}
    Follows from exponential tilting of the reward-weighted distribution \citep{rwr}.
\end{remark}

\textbf{Limitation:} All intermediate states $x_t$ in a trajectory share the same weight $w\left(x_1, c\right)$, leading to ambiguous credit assignment.

\subsubsection{Advantage-Weighted Regression (AWR)}

Advantage-Weighted Regression \citep[AWR]{awr} refines RWR by leveraging the advantage function $A^\pi\left(x_t, t, c\right)=Q^\pi\left(x_t, t, c\right)-V^\pi\left(x_t, t, c\right)$, which measures how much an action outperforms the average at state $x_t$. The weight becomes:

\begin{equation}
    w_{\mathrm{AWR}}\left(x_t, c\right) \propto \exp \left(\tau A^\pi\left(x_t, t, c\right)\right) .
\end{equation}

The actor objective transitions to:

\begin{equation}
    \mathcal{L}_{\mathrm{AWR}}(\theta)=\mathbb{E}_{t, x_t, c, x}\left[w_{\mathrm{AWR}}\left(x_t, c\right)\left\|v_\theta(t, x, c)-u_t\left(x \mid x_t, c\right)\right\|^2\right]
\end{equation}

\begin{theorem}[(AWR Variance Reduction)]
Let $\operatorname{Var}_{R W R}$ and $\operatorname{Var}_{A W R}$ denote gradient variances under RWR and AWR, respectively. Then:
\begin{equation}
\operatorname{Var}_{\mathrm{AWR}} \leq \operatorname{Var}_{\mathrm{RWR}}
\end{equation}
\end{theorem}

\begin{remark}
    Advantage normalization reduces the variance of the importance weights \citep{gae}.
\end{remark}

Benefits: 1. Localized credit assignment via advantage estimates. 2. Mitigates overfitting to high-reward outliers. 

However, how to estimate the advantage function that includes intermediate states value estimation while fine-tuning flow matching models has not yet been widely studied. Most previous works either adopt final outcome reward weighting/selection \citep{ddpo,Raft,ReFT} or recent widely used group relative advantage estimation \citep{grpo}.

\subsubsection{Group Relative Advantage Estimation (GRAE)}

Group Relative Advantage Estimation \citep[GRAE]{grpo} extends AWR by normalizing advantages across trajectories to stabilize training. For a batch of $M$ trajectories, compute:

\begin{equation}
    A_{\mathrm{GRAE}}\left(x_t, t, c\right)=\frac{A^\pi\left(x_t, t, c\right)-\mu_A}{\sigma_A}
\end{equation}

where $\mu_A, \sigma_A$ are the mean and standard deviation of advantages in the batch. The weight is:

\begin{equation}
    w_{\mathrm{GRAE}}\left(x_t, c\right) \propto \exp \left(\tau A_{\mathrm{GRAE}}\left(x_t, t, c\right)\right)
\end{equation}

\begin{theorem}[GRAE Training Stability]
    Under Lipschitz continuity of $\boldsymbol{A}^\pi$, GRAE ensures bounded policy updates:

\begin{equation}
\left\|\nabla_\theta \log \pi_\theta\right\| \leq C \cdot \tau / \sigma_A
\end{equation}

for constant $C>0$.
\end{theorem}

\subsubsection{Policy Gradient Methods}

Policy Gradient (PG) \citep{rlsutton,gae,lbc,prance,sp_vla} methods directly optimize the expected return $J\left(\pi_\theta\right)=\mathbb{E}_\pi\left[r\left(x_1, c\right)\right]$ via gradient ascent:

\begin{equation}
    \nabla_\theta J\left(\pi_\theta\right)=\mathbb{E}_\pi\left[\nabla_\theta \log \pi_\theta\left(a_t \mid x_t, t, c\right) \cdot A^\pi\left(x_t, t, c\right)\right]
\end{equation}

\begin{theorem}[Intractability of Policy Gradients]
    Flow matching models parameterize policies as continuous-time ODEs:

\begin{equation}
x_1=x_0+\int_0^1 v_\theta\left(t, x_t, c\right) d t
\end{equation}

The gradient $\nabla_\theta \log \pi_\theta$ requires computing the divergence $\nabla_x \cdot v_\theta$, which is \#P-hard for high-dimensional $x$.
\end{theorem}

\begin{theorem}[PG Inapplicability]
    Exact policy gradient computation for flow matching is intractable under standard complexity-theoretic assumptions.
\end{theorem}

\begin{remark}
    Follows from the equivalence of $\nabla_x \cdot v_\theta$ to matrix permanent computation \citep{Complexity_of_Compute}.
\end{remark}

\subsection{General Policy Gradient Framework}

Our GCW formation is inspired from the unified  policy gradient framework proposed in GAE paper \citep{gae}. A widely used expression for the policy gradient is:

\begin{equation}
    \nabla_\theta J(\theta)=\mathbb{E}_{\tau \sim \pi_\theta}[\sum_{t=0}^{T-1} \underbrace{\nabla_\theta \log \pi_\theta\left(a_t \mid s_t\right)}_{\text {policy log-derivative }} \cdot \underbrace{\Psi_t}_{\text {weighting signal }}]
\end{equation}

Here, $\Psi_t$ is a scalar \emph{weighting signal}that drives the gradient updates (i.e., $T \to \infty$ becomes continues-time \citep{xiangxin_stable}). Different choices of $\Psi_t$ correspond to different RL algorithms (Sparse rewards, assuming only $r^T= r(x_1,c)$ can be obtained from reward models):

\textbf{1. Outcome Reward:} $\Psi_t=\sum_{k=t}^T r_k=r^T=r(x_1,c)$ (REINFORCE).

\textbf{2. Action-Value:} $\Psi_t=Q^\pi\left(s_t, a_t\right)$.

\textbf{3. Advantage:} $\Psi_t=A^\pi\left(s_t, a_t\right)=Q^\pi\left(s_t, a_t\right)-V^\pi\left(s_t\right)$ .

By carefully designing $\Psi_t$, one can reduce variance (e.g., via baselines or advantage functions) and facilitate stable training. In discrete-time RL tasks, these ideas have proven extremely successful in policy optimization.

\clearpage






\clearpage

\section{Hyper-Parameter}
\label{app: hyper-paramter}

The performance and stability of AC-Flow hinge on several key hyperparameters that guide learning dynamics and influence the final generative distribution. This section provides a thorough analysis of these parameters, supported by ablation studies and theoretical insights into how various components of the system interact.

\subsection{Policy Learning Parameters}

\subsubsection{Temperature Scaling}

The temperature parameter $(\tau)$ governs the sharpness of the generalized critic weighting function, thereby controlling how advantage estimates influence policy updates. 
Through extensive empirical testing, we find that $\tau=1.0$ offers the best trade-off. Larger values ( $\tau>2.0$ ) produce aggressive policy updates that can destabilize training, whereas smaller values ( $\tau<0.5$ ) yield overly cautious policy shifts and slower convergence.

\subsubsection{Wasserstein Regularization}

Wasserstein regularization strength ( $\alpha=1.0$ ) is critical in balancing distribution diversity against the drive for policy improvement. It modifies the policy objective with an additional term $\alpha \Omega\left(\theta ; \theta_{\text {ref }}\right)$. Stronger regularization ( $\alpha>2.0$ ) keeps the distribution closer to the reference but may hinder improvement, while weaker regularization ( $\alpha<0.5$ ) permits larger distribution shifts but increases the risk of mode collapse.

\subsection{Stability Control Mechanisms}

\subsubsection{Advantage Estimation and Clipping}

To prevent extreme policy updates while preserving a meaningful learning signal, we employ advantage clipping with a threshold $\delta=5.0$. This clipping ensures gradients remain within manageable bounds, contributing to training stability. Empirical analysis indicates that $\delta=5.0$ strikes an optimal balance between robust learning and controlled updates.
\subsubsection{Reward Normalization}

Reward normalization with a small $\epsilon\left(\epsilon=10^{-6}\right)$ is used during min-max scaling:

\begin{equation}
A=\frac{r-\min (r)}{\max (r)-\min (r)+\epsilon}
\end{equation}

This procedure stabilizes critic updates by maintaining consistent gradient magnitudes. The choice of $\epsilon=10^{-6}$ safeguards numerical stability while still allowing fine-grained reward distinctions.

\subsection{Training Dynamics}

\subsubsection{Warm-up Period}

A warm-up phase of $k=500$ steps precedes the switch to full critic-based advantages, during which group-relative advantage estimation (GRAE) is employed \citep{grpo}. This delay helps the critic acquire reliable value estimates before it significantly influences policy updates. Shorter warm-ups ( $k<$ 200) often result in unstable early training, whereas prolonged warm-ups ( $k>1000$ ) unnecessarily delay effective learning.

\subsubsection{Learning Rates}

Differential learning rates are used for the actor $\left(\eta_a=10^{-4}\right)$ and the critic $\left(\eta_c=3 \times 10^{-4}\right)$, both leveraging the AdamW optimizer. The slightly higher rate for the critic ensures timely and accurate value estimation as the policy evolves, while the more conservative rate for the actor avoids excessive swings in policy space.

\subsubsection{Batch Size}

We adopt a batch size of $B=256$ for both actor and critic training. Larger batches ( $B>512$ ) can lower gradient variance but impose greater computational overhead, whereas smaller batches ( $B<128$ ) benefit round speed at the cost of higher gradient variance. We train both the policy and the critic network using AdamW with gradient clipping (norm $\leq 1.0$ ) and a learning rate of $1 e-4$ for the policy, while matching or slightly increasing it for the critic to ensure robust value estimation throughout training.

\subsection{Parameter Interactions}

A key strength of AC-Flow lies in how these hyperparameters function together:
\begin{enumerate}
    \item Wasserstein Regularization vs. Batch Size: More robust regularization ( $\alpha$ ) often requires larger batches $(B)$ to maintain stability.
   
    \item Warm-up ( $k$ ) vs. Critic Learning Rate $\left(\eta_c\right)$ : Sufficient warm-up is vital for establishing reliable critic estimates before intensifying policy updates.
    
\item Temperature ( $\tau$ ) vs. Advantage Clipping ( $\delta$ ): These jointly moderate the magnitude of policy updates, necessitating careful co-tuning to avoid instability or slow learning.
\end{enumerate}

Most of the hyperparameters in our method have a very intuitive impact on the training process and results. In practice, we can adjust the hyperparameters of our algorithm as needed to control the convergence behavior of the model and customize the converged solution to achieve reward-diversity trade-off.

\clearpage

\section{Experimental Details}
\label{app: Experimental Details}


In this section, we provide comprehensive details about our experimental setup, implementation, and evaluation protocols to ensure reproducibility of our results.

\subsection{Implementation Details}

As Fig. \ref{fig: general framework}, our implementation is based on the SD3 \citep{sd3} architecture with several key modifications to accommodate our actor-critic framework. The actor network utilizes the MM-DIT architecture from SD3 with additional LoRA adaptation layers for parameter-efficient fine-tuning \citep{lora}.

Also as Fig. \ref{fig: general framework}, for the critic network, we design a computationally efficient architecture to enable efficient value estimation while maintaining computational tractability. As Fig. \ref{fig: general framework}, the critic comprises several layers of MLP \citep{pal1992multilayer} with Layer Normalization \citep{layernorm} and CNN \citep{cnn_paper}. This design allows the critic to effectively process high-dimensional image states while maintaining reasonable memory and computational requirements without introducing too much computation overhead \citep{grpo}.

\subsection{Computational Resources}

Each experiment was conducted using \textbf{a single NVIDIA RTX A6000 GPU} (48GB variant). A complete training run requires approximately 24-36 hours on our hardware configuration. To make the fine-tuning of Stable Diffusion 3 \citep{sd3} feasible on widely available hardware, we adopt several efficiency-focused design choices: LoRA \citep{lora} for parameter-efficient adaptation, float-16 precision for reduced memory footprint, and a computationally efficient Critic Network comprising several layers of CNN and MLP. \textbf{These optimizations makes our method accessible for research and development on single-GPU setups.}

\subsection{Baseline Methods}
For comprehensive evaluation, we implement and compare against several state-of-the-art methods:

RAFT \citep{Raft} is implemented following the official codebase, maintaining identical batch sizes and computational budgets as our method. ReFT \citep{ReFT} required adaptation for SD3 compatibility while preserving the core algorithm design. For RWR+W2 \citep{rwr,iclr_rwr}, we implement the Wasserstein-regularized version, matching our method's regularization scheme. AWR+W2  \citep{awr} is extended with Wasserstein regularization and group advantage estimation (GRAE) \citep{grpo}.

\subsection{Reproducibility Considerations}

To ensure full reproducibility of our work, we have taken several comprehensive measures. First, we will release our complete codebase upon publication, including training scripts, evaluation pipelines, and model implementations. The general framework of our approach is thoroughly documented in Fig. \ref{fig: general framework}, which provides detailed architectural diagrams of both the actor and critic components, along with comprehensive illustrations of our key training techniques including reward shaping, advantage clipping, and warm-up strategies.

For precise replication of our experimental setup, we provide exhaustive hyperparameter specifications in App. \ref{app: hyper-paramter}. This includes not only the primary parameters  but also implementation-specific details such as learning rates, batch sizes, and optimization settings. The interaction between these hyperparameters and their impact on model stability is discussed in detail to guide practical implementations.

The complete training protocol is formalized in Algorithm \ref{alg: ac_flow} (App. \ref{app: Ac-Flow Algorithm}), which presents detailed pseudocode for our AC-Flow framework. This algorithm explicitly outlines each step of the training process, from data sampling and critic updates to actor optimization and stability control mechanisms. We also include specific numerical stability considerations and practical tips for implementation that we found crucial during development.

To facilitate fair comparison with our baselines, we document the exact evaluation metrics used for each metric (CLIP Score \citep{clip}, HPS v2 \citep{hpsv2}, ImageReward \citep{image_reward}, Diversity Score \citep{iclr_rwr}) in App. \ref{app: Evaluation Metrics}.

Our memory-efficient implementation, leveraging LoRA adaptation and float-16 precision, is thoroughly documented to enable reproduction on widely available hardware. The computationally efficient critic architecture details and memory optimization strategies are specified to achieve the reported training in single A6000 GPU.

\clearpage

\section{Evaluation Metrics}
\label{app: Evaluation Metrics}

In this section, we provide detailed descriptions of the evaluation metrics used in our experimental analysis. Our evaluation framework employs multiple complementary metrics to assess both the quality and diversity of generated images, as well as their alignment with text prompts.

\subsection{CLIP Score}

CLIP Score \citep{clip} measures the semantic alignment between generated images and their corresponding text prompts using the CLIP model's cross-modal embedding space. For an image-text pair $(x, c)$, the CLIP Score is computed as:

\begin{equation}
    \operatorname{CLIP} \_\operatorname{Score}(x, c)=100 \times \cos \left(E_{\mathrm{img}}(x), E_{\mathrm{txt}}(c)\right),
\end{equation}

where $E_{\mathrm{img}}$ and $E_{\mathrm{txt}}$ are CLIP's image and text encoders respectively, and $\cos (\cdot, \cdot)$ denotes cosine similarity. Higher CLIP Scores $(\uparrow)$ indicate better semantic alignment between the generated images and text prompts. We use this metric as our primary measure of text-image alignment quality, as it has been shown to correlate well with human judgments of semantic consistency.

\subsection{HPS v2 (Human Preference Score v2)}

HPS v2 \citep{hpsv2}, is a learned metric that aims to predict human preferences for text-to-image generation outputs. It was trained on a large-scale dataset of human preference judgments, incorporating multiple aspects of image quality including visual fidelity, text alignment, and aesthetic appeal. The score ranges from 0 to 100, with higher values $(\uparrow)$ indicating stronger predicted human preference. HPS v2 is particularly valuable for our evaluation as it provides a more holistic assessment of generation quality that goes beyond simple text-image alignment.

\subsection{ImageReward}

ImageReward \citep{image_reward} is a specialized reward model trained to evaluate text-to-image generation outputs by directly learning from human feedback. It employs a transformer-based architecture to compute a scalar reward value that captures both image-text alignment and image quality. The metric is normalized to $[0,1]$, with higher values $(\uparrow)$ indicating better generation quality. ImageReward complements our evaluation suite by providing an additional perspective on human preference prediction that was trained using a different methodology than HPS v2.

\subsection{Diversity Score}

We measure sample diversity using pairwise distances between CLIP image embeddings within a batch of generations for the same prompt. For a batch of $n$ generated images $\left\{x_1, \ldots, x_n\right\}$, the Diversity Score is computed as:

\begin{equation}
    \text { Diversity Score }=\operatorname{mean}_{i \neq j}\left\|E_{\text {img }}\left(x_i\right)-E_{\text {img }}\left(x_j\right)\right\|_2,
\end{equation}

where $E_{\text {img }}$ is the CLIP image encoder. Higher values $(\uparrow)$ indicate greater diversity among generated samples. This metric is crucial for assessing whether our method maintains generative diversity while improving quality, helping us detect potential mode collapse issues that are common in fine-tuning approaches.

\subsection{Metric Complementarity}

Our choice of metrics is deliberately designed to capture different aspects of generation quality:

\begin{itemize}
    \item CLIP Score focuses on semantic alignment. (widely-used reward model for text-image alignment)
    \item HPS v2 and ImageReward provide human-aligned quality assessment. (generalization ability)
    \item Diversity Score ensures maintenance of sample variation. (diversity preservation, indicator for model collapse, entropy collapse \citep{cn_rl_entropy,gdi,varcon})
\end{itemize}

This comprehensive evaluation framework allows us to assess both the improvements in generation quality and potential trade-offs in our approach. The combination of these metrics provides a more complete picture of model performance than any single metric alone.

\clearpage

\section{AC-Flow Algorithm}
\label{app: Ac-Flow Algorithm}

\begin{algorithm}
\caption{AC-Flow: Actor-Critic Framework for Flow Matching}
\label{alg: ac_flow}
\begin{algorithmic}[1]
\REQUIRE Reference model $\theta_{\text{ref}}$, critic parameters $\phi$, temperature $\tau$, regularization weight $\alpha$, advantage clip threshold $\delta$, warm-up steps $k$

\STATE Initialize $\theta \gets \theta_{\text{ref}}$
\vspace{1mm}

\WHILE{not converged}
    \vspace{1mm}
    \STATE // Generate samples using current policy
    \STATE Sample $x_0 \sim p_0$, context $c \sim p(c)$
    \STATE Solve ODE: $x_1 = x_0 + \int_0^1 v_{\theta}(t,x_t,c)dt$
    \STATE Compute reward $r(x_1, c)$ using reward model
    \STATE Sample $t \sim \mathcal{U}[0,1]$, obtain $x_t=tx_1+(1-t)x_0, u_t=x_1-x_0$ (e.g., linear interpolation).
    \vspace{1mm}
    
    \STATE // ---------- CRITIC UPDATES ----------
    \vspace{1mm}
    \STATE \textbf{Stable Intermediate Value Evaluation:}
    \vspace{1mm}
    
    \STATE // 1. Reward Shaping for stability
    \STATE $\tilde{r} = \frac{r-\min(r)}{\max(r)-\min(r)+\epsilon}$
    \vspace{1mm}
    
    \STATE // 2. Value Function Updates
    \STATE \textbf{Critic Loss:} $\mathcal{L}_{\text{critic}} = (V_\phi(x_t,t,c) - \tilde{r})^2$
    \STATE $\phi \gets \phi - \eta_c\nabla_\phi\mathcal{L}_{\text{critic}}$
    \vspace{2mm}
    
    \STATE // ---------- ACTOR UPDATES ----------
    \vspace{1mm}
    \STATE \textbf{Robust Actor-Critic Framework:}
    \vspace{1mm}
    
    \IF{steps $\leq k$}
        \STATE // 3. Critic Warm-Up Phase
        \STATE $\mu_{\mathcal{B}} = \text{mean}(\tilde{r})$, $\sigma_{\mathcal{B}} = \text{std}(\tilde{r})$
        \STATE $A = \frac{\tilde{r}-\mu_{\mathcal{B}}}{\sigma_{\mathcal{B}}+\epsilon}$ // GRAE
    \ELSE
        \STATE // Critic-based advantage estimation
        \STATE $A = \tilde{r} - V_\phi(x_t,t,c)$
    \ENDIF
    \vspace{1mm}
    
    \STATE // 4. Advantage Clipping for stability
    \STATE $A_{\text{clip}} = \text{clip}(A, -\delta, \delta)$
    \vspace{2mm}
    
    \STATE \textbf{Generalized Critic Weighting:}
    \vspace{1mm}
    
    \STATE // Calculate weight based on advantage
    \STATE $w = \exp(\tau \cdot A_{\text{clip}})$
    \vspace{2mm}
    
    \STATE \textbf{Wasserstein Regularization:}
    \vspace{1mm}
    
    \STATE // Compute W2 distance approximation
    \STATE $\Omega = \|v_\theta(t,x_t,c) - v_{\theta_{\text{ref}}}(t,x_t,c)\|^2$
    \vspace{1mm}
    
    \STATE \textbf{Actor Loss:} $\mathcal{L}_{\text{actor}} = w \cdot \|v_\theta(t,x_t,c) - u_t(x_t|x_1,c)\|^2 + \alpha \cdot \Omega$
    \STATE $\theta \gets \theta - \eta_a\nabla_\theta\mathcal{L}_{\text{actor}}$
    \vspace{1mm}
\ENDWHILE
\vspace{1mm}

\RETURN Optimized model parameters $\theta$
\end{algorithmic}
\end{algorithm}

\clearpage

\section{Additional Experimental Results}
\label{app: Additional Experimental Results}
\subsection{Ablation  Study of Proposed Stabilization Techniques}

Our empirical analysis reveals that naive direct value regression, without stabilization mechanisms, exhibits severe actor-critic loss instability during early training phases. This instability manifests not only in the loss landscape but also in the reward dynamics. Specifically, when an insufficiently trained critic is used to compute advantages and guide policy updates, it triggers a destructive feedback loop that ultimately leads to training collapse \citep{offline_to_online}. As illustrated in Fig. \ref{fig: reward curve of case study}, the reward curve for the prompt ``A dog in the sky'' demonstrates how the baseline model's performance steadily deteriorates over time. This example vividly illustrates the limitations of vanilla online actor-critic methods in flow matching - \textbf{without proper stabilization techniques}, the framework \textbf{fails to maintain stable continual improvement}, resulting in catastrophic performance degradation. The contrast between the baseline's declining rewards and our stabilized approach's steady improvement underscores the critical importance of our proposed stabilization mechanisms.

\begin{figure*}[!ht]
	\centering
 \includegraphics[width=0.5\linewidth]{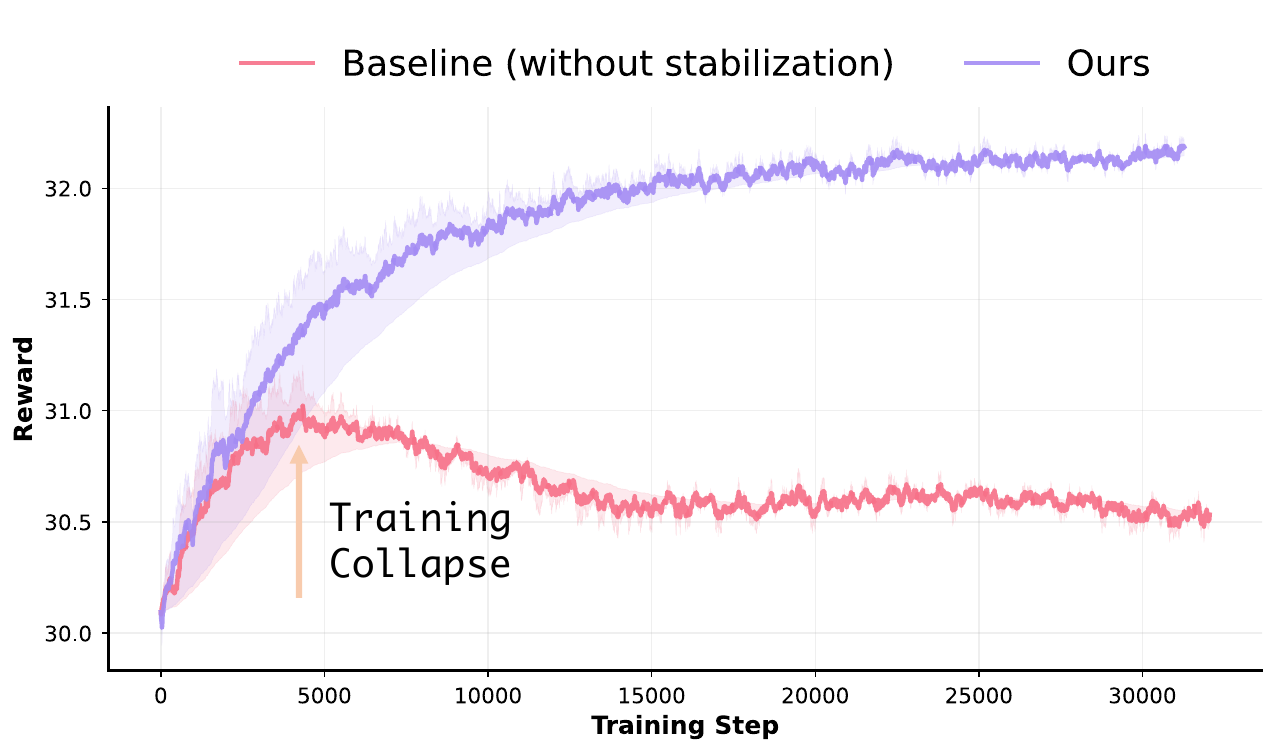}
	\caption{Reward Curve in Case Study of Proposed Stabilization Techniques.}
    \label{fig: reward curve of case study}
\end{figure*}

\subsection{Computational Cost}

\begin{table}[!ht]
\caption{Computational Cost of Different Methods.}
\label{tab:com_time}
 \begin{center}
     \resizebox{0.8\textwidth}{!}{
\begin{tabular}{ccc}
\toprule
\textbf{Algorithm} & \textbf{Run Time (Hours)} & \textbf{Peak GPU Memory (GB)} \\\\
\midrule
AC-Flow & 30.76 & 42 \\
AC-Flow (w/o W2) & 25.79 & 31 \\
RWR & 28.51 & 40 \\
AWR+W2 & 28.72 & 40 \\
RAFT & 24.42 & 29 \\
ReFT & 24.63 & 29 \\
\bottomrule
\end{tabular}
}
 \end{center}
\end{table}

Based on Tab. \ref{tab:main_results} and \ref{tab:com_time}, our experimental analysis shows that while AC-Flow achieves state-of-the-art performance across multiple metrics, it does require slightly higher computational resources compared to simpler baseline methods. As presented in Table \ref{tab:com_time}, AC-Flow requires approximately 30.76 hours of training time on a single NVIDIA RTX A6000 GPU with a peak memory usage of 42GB. This increased resource requirement is primarily attributed to the additional computations needed for our actor-critic framework with intermediate state evaluation.

The ablated version without Wasserstein regularization (AC-Flow w/o W2) demonstrates a notable reduction in computational demands, requiring 25.79 hours and 31 GB of peak memory, highlighting the computational cost of maintaining generative diversity. When compared to other fine-tuning approaches, our method shows comparable efficiency to reward-weighted methods like RWR (28.51 hours, 40GB) and AWR+W2 (28.72 hours, 40 GB ), while requiring moderately more resources than simpler methods such as RAFT (24.42 hours, 29GB) and ReFT (24.63 hours, 29GB).

This computational profile represents a reasonable trade-off between performance gains and resource requirements, with AC-Flow delivering superior semantic alignment and human preference scores while maintaining acceptable training efficiency for research environments equipped with modern GPUs.

\subsection{Additional Alignment Results}

\begin{figure*}[!ht]
	\centering
 \includegraphics[width=\linewidth]{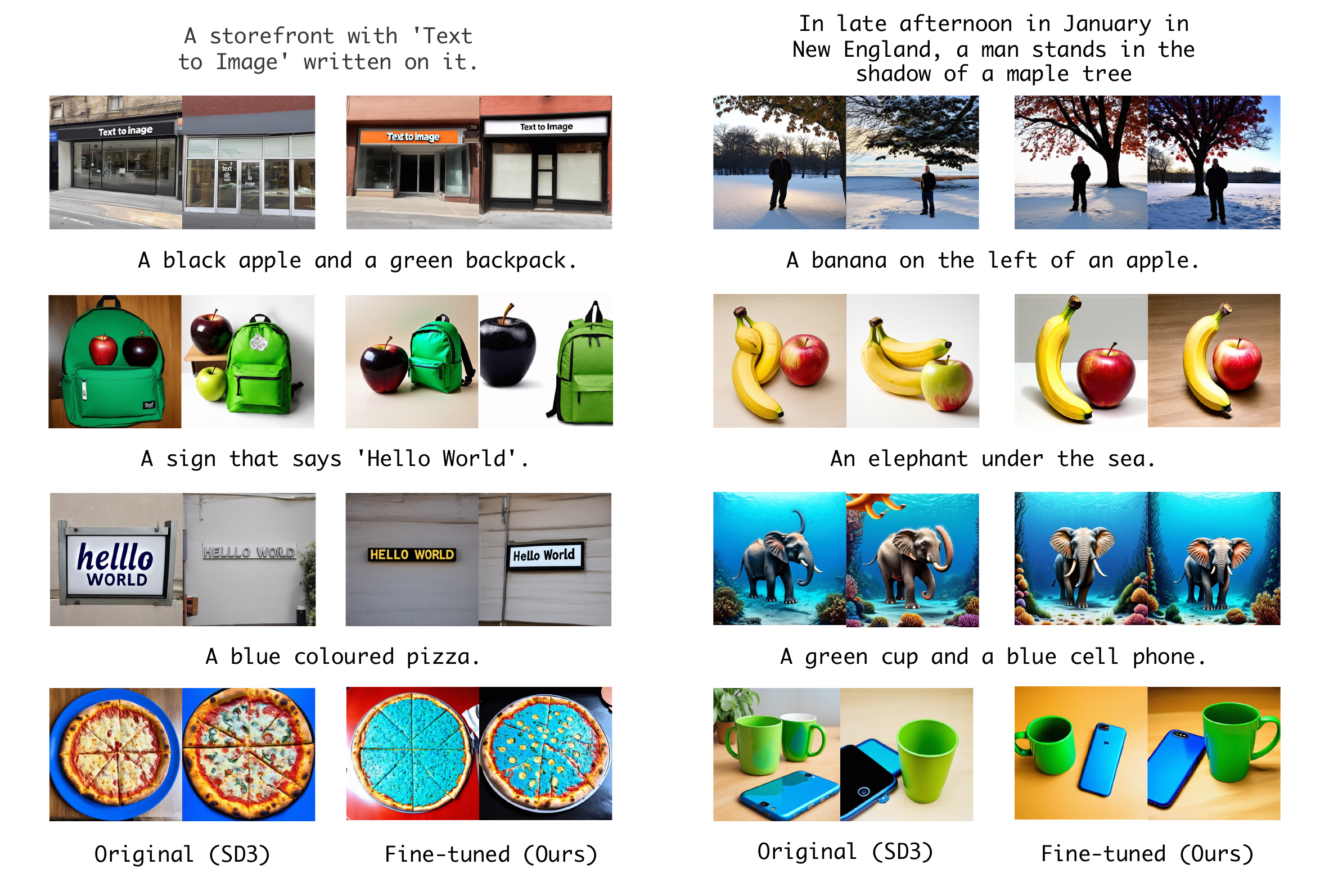}
	\caption{Complex Semantic Alignment Results. Comparison between original base model SD3 (left column pairs) and our fine-tuned model (right column pairs) on challenging text-to-image generation tasks. Our method demonstrates superior performance in handling spatial relationships (banana-apple positioning), attribute binding (black apple, green backpack), environmental context (New England winter scene), and unconventional concepts (blue pizza, underwater elephant) while maintaining visual quality and semantic coherence. Each test case evaluates different aspects of semantic understanding and compositional generation capabilities.}
    \label{fig: main complex semantic alignment}
\end{figure*}

Here, we evaluate our method's ability to handle challenging text-to-image generation tasks that require precise semantic alignment and spatial relationships. As shown in Figure \ref{fig: main complex semantic alignment}, our fine-tuned model demonstrates superior performance across diverse prompts involving complex spatial directives (``banana on the left of an apple''), attribute binding (``black apple and green backpack''), and contextual understanding (``late afternoon in January in New England''). The results highlight two key strengths of our approach: First, our actor-critic framework with intermediate state evaluation enables more precise control over the generation process, allowing the model to accurately capture spatial relationships and attribute bindings while maintaining visual quality. Second, the combination of advantage estimation and Wasserstein regularization helps prevent semantic drift and attribute mixing, even in challenging scenarios like ``blue colored pizza'' and ``elephant under the sea'' where the model must balance semantic alignment with visual plausibility. These qualitative improvements align with our quantitative results in Table \ref{tab:main_results}, demonstrating that our method achieves better semantic alignment without sacrificing generation diversity.

\end{document}